\documentclass[10pt,a4paper]{article}
\usepackage[utf8]{inputenc}
\usepackage[margin=1in]{geometry}
\usepackage[titletoc,title]{appendix}
\usepackage{amsmath,amsfonts,amssymb,mathtools}
\usepackage{graphicx,float}
\usepackage{times}
\usepackage{mlmath}
\usepackage{color}
\usepackage{hyperref}
\usepackage{enumerate}
\usepackage{amsthm}
\usepackage{geometry}
\newtheorem{thm}{Theorem}[section]
\newtheorem{coro}[thm]{Corollary}
\newtheorem{defi}[thm]{Definition}
\newtheorem{lem}[thm]{Lemma}
\newtheorem{prop}[thm]{Proposition}
\newtheorem{remark}[thm]{Remark}
\newtheorem{assum}[thm]{Assumption}

\title{Reinforcement learning}
\author{Zhou Qixuan}
\date{March 2022}

\title{A priori Estimates for Deep Residual Network in Continuous-time Reinforcement Learning}
\author{
    Shuyu Yin \footnotemark[1] \, \footnotemark[3]
    \and
    Qixuan Zhou \footnotemark[1] \, \footnotemark[4]
    \and
    Fei Wen \footnotemark[3]
    \and
    Tao Luo \footnotemark[2] \, \footnotemark[5] \, \footnotemark[6]
}

\date{\today}
\begin{document}
\maketitle

\renewcommand{\thefootnote}{\fnsymbol{footnote}}
\footnotetext[1]{Equal Contribution}
\footnotetext[2]{Corresponding Author: luotao41@sjtu.edu.cn} 
\footnotetext[3]{Department of Electronic Engineering, Shanghai Jiao Tong University}
\footnotetext[4]{School of Mathematical Sciences, Shanghai Jiao Tong University}
\footnotetext[5]{School of Mathematical Sciences, Institute of Natural Sciences, MOE-LSC, CMA-Shanghai, Shanghai Jiao Tong University} 
\footnotetext[6]{Shanghai Artificial Intelligence Laboratory}

\begin{abstract}
Deep reinforcement learning excels in numerous large-scale practical applications. However, existing performance analyses ignores the unique characteristics of continuous-time control problems, is unable to directly estimate the generalization error of the Bellman optimal loss and require a boundedness assumption. Our work focuses on continuous-time control problems and proposes a method that is applicable to all such problems where the transition function satisfies semi-group and Lipschitz properties. Under this method, we can directly analyze the \emph{a priori} generalization error of the Bellman optimal loss. The core of this method lies in two transformations of the loss function. To complete the transformation, we propose a decomposition method for the maximum operator. Additionally, this analysis method does not require a boundedness assumption. Finally, we obtain an \emph{a priori} generalization error without the curse of dimensionality.
\end{abstract}

\textbf{Keywords:} \emph{A priori} Estimates, Residual Network, Continuous-time Reinforcement Learning, Bellman Optimal Loss. 

\section{Introduction}\label{sec::Intro}

Reinforcement learning methods, characterized by their consecutive interaction with the environment and learning from feedback, are adept at solving a wide range of sequential decision-making and control problems. Deep reinforcement learning  (DRL), which employs neural networks to represent value or policy functions, utilizes common algorithms such as DQN~\cite{mnih2015human}, PPO~\cite{schulman2017proximal}, TRPO~\cite{schulman2015trust}, and SAC~\cite{haarnoja2018soft}. Deep reinforcement learning has proven its capability to outperform  human-designed algorithms in numerous practical applications. These include discrete-time decision-making problems like Go~\cite{silver2016mastering}, recommendation systems~\cite{afsar2022reinforcement}, and combinatorial optimization problems~\cite{mazyavkina2021reinforcement}. Furthermore, deep reinforcement learning is also applicable to continuous-time control problems, such as robotic control~\cite{han2023survey}, UAV control~\cite{koch2019reinforcement}, Starcraft~\cite{vinyals2019grandmaster}, and quantitative trading~\cite{sun2023reinforcement}. To accommodate the reinforcement learning method, continuous-time control problems often necessitate time discretization.

While there are existing studies that analyze the theoretical performance of DRL, including aspects such as generalization error and regret~\cite{duan2021risk, long2021an, fan2020theoretical, yang2020function, wang2020provably, nguyen2021sample}, three significant gaps remain. First, existing works treat continuous-time and discrete-time control problems in a uniform manner, failing to utilize the characteristics of continuous-time control and to consider the actual impacts and errors introduced by the time step of discretization. Specifically, in continuous-time control problems, MDPs often exhibit the smooth policy property, implying that actions and state values for closely located states should be similar. Agents satisfy this property have better policies~\cite{shen2020deep}, however, current works often overlook this property. Moreover, applying reinforcement learning methods to continuous-time control problems requires discretization, and the choice of the time step of discretization is crucial. On the one hand, large step can lead to greater generalization error, resulting in less precise control and drastic action changes in real-world problems. On the other hand, small step can lead to excessive computational resource usage and increased costs. Current works can not guide us in choosing the appropriate step size. Second, existing works are based on surrogate loss functions and cannot directly provide the generalization error for the Bellman optimal loss. Existing works mainly estimate the generalization error for iterative algorithms like Fitted Q-Learning (FQI,~\cite{riedmiller2005neural}) and Least Squares Value Iteration (LSVI,~\cite{sutton2018reinforcement}), which construct surrogate losses based on the target network. Using surrogate losses can avoid dealing with the maximum operation, but the surrogate losses come in various forms, such as double Q-Learning~\cite{van2016deep}, Rainbow~\cite{hessel2018rainbow}, and existing analytical methods do not cover them. Furthermore, optimizing these surrogate losses ultimately aims to optimize the Bellman optimal loss. Hence, it is essential to directly estimate the generalization error of the Bellman optimal loss. Third, in existing generalization error estimations, it is often assumed that the function used for approximation is bounded, or that the values of the target network are bounded. However, this assumption is not commonly applied in practical situations. Therefore, it is necessary to eliminate the boundedness assumption.

In this work, we estimate the generalization error for the Bellman optimal loss in contin-uous-time control problems, which have a continuous state space and a finite action space. We utilize a multi-layer residual network, coupled with explicit regularization, to approximate the Q-function. To concentrate on the error bound analysis of function approximation, we take an assumption that the data in the empirical loss function is uniformly distributed. In addition, we introduce a discrete-time transition function and assume this function possesses semi-group and Lipschitz properties, which allows the value function to satisfy the smooth policy property. Moreover, we assume that the reward function is bounded and resides in the Barron space. This assumption is plausible as reward functions are typically custom-designed and we only require the reward function to be continuous with bounded values, a broad condition that most reward functions can meet.

In the overall analysis, the most challenging step is obtaining the approximation error of the Bellman optimal loss. To estimate this error bound, we formulate an Bellman effective loss by setting the time step of discretization to zero, which can be viewed as an extreme configuration. we are able to derive the explicit solution of Bellman effective loss, which is expressed using the reward function. Then we can employ a neural network to directly fit it through supervised loss and estimate the approximation error. Moreover, since the Bellman effective loss is bounded by the previously obtained supervised loss, we are able to obtain the approximation error for the Bellman effective loss. Finally, thanks to the Lipschitz property of the discrete-time transition function, we express the Bellman optimal loss as the sum of the Bellman effective loss and an additional term related to the Lipschitz constant and the time step of discretization. This allows us to obtain the approximation error of the Bellman optimal loss. 

\subsection{Contribution}

% 我们的方法适用于所有transiiton function满足半群和Lipstize条件的连续时间控制问题，往往在这类问题环境中由强化学习训练得到的agent性能更好。

Our primary contributions can be encapsulated in three key aspects: \textbf{1) More Realistic Settings for Continuous-time Control Problems:} Our method is applicable to all continuous-time control problems where the transition function satisfies semigroup and Lipschitz conditions. In such problem environments, agents trained by reinforcement learning often exhibit better performance~\cite{shen2020deep}. \textbf{2) Generalization Error Estimation for Bellman Optimal Loss:} Our method can directly handle the Bellman optimal loss, which is based on our new approach. This approach involves transforming the loss function twice: from the supervised loss of the effective solution to the Bellman effective loss, and from the Bellman effective loss to the Bellman optimal loss. Specifically, in the first transformation, we introduce a decomposition technique to deal with the maximum operator, which is often challenging to analyze. Additionally, we use the generalization error bound to guide the choice of time step of discretization. \textbf{3) Eliminated Boundedness Assumption:} Our analysis does not require the boundedness assumption, and our result does not suffer from the curse of dimensionality. The error bound relies only on the polynomial of the action space size, sample size, and neural network width.

\subsection{Related works}

This work is intimately linked to the body of literature on batch reinforcement learning, where the goal is to estimate the value function based on a provided sample dataset. This problem typically optimize a Bellman (optimal) loss in a least-square format. Because of the complexity of the neural networks, most existing works only consider using linear function to approximation the value function~\cite{boyan1999least,bradtke1996linear,farahmand2010error,lazaric2012finite,lazaric2016analysis,tagorti2015rate,lagoudakis2003least}. However, our work focuses on using neural networks to approximate the value function. Existing works under this setting generally consider FQI~\cite{duan2021risk, long2021an, fan2020theoretical, nguyen2021sample} and LSVI~\cite{yang2020function, wang2020provably} algorithms. In the context of the FQI algorithm, Fan \emph{et al.}~\cite{fan2020theoretical} provide a generalization error with bounded multi-layer fully connected neural networks, incorporating the phenomenon of distribution shift. Duan \emph{et al.}~\cite{duan2021risk} acknowledge the variance term in the empirical Bellman optimal loss and discuss the generalization error in the single sample regime under bounded functions. Similarly, Nguyen \emph{et al.}~\cite{nguyen2021sample} recognize the variance term and use uniform data convergence to mitigate it, thereby obtaining a generalization error under bounded multi-layer fully connected neural networks. Our work draws primary inspiration from Long \emph{et al.}~\cite{long2021an}, who introduce an explicit regularization term into the empirical loss function and achieve a generalization bound with two-layer neural networks, the neural networks utilized for approximation are unbounded, however, the target network necessitates a truncation operator. In the context of the LSVI algorithm, Yang \emph{et al.}~\cite{yang2020function} provide a generalization error using an Upper Confidence Bound (UCB,~\cite{sutton2018reinforcement}) sampling mechanism. Wang \emph{et al.}~\cite{wang2020provably} further enhance the UCB mechanism by integrating an effective sampling technique~\cite{langberg2010universal, feldman2011unified, feldman2020turning}, thereby obtaining a generalization bound under a general function space with bounded functions. Similar to the method employed by Long \emph{et al.}~\cite{long2021an}, our work also makes use of explicit regularization. However, we diverge in that we directly estimate the generalization error for the Bellman optimal loss, and our approach does not necessitate a boundedness assumption for the target network.
%~\cite{fan2020theoretical} provides a generalization error under bounded multi-layer fully connected neural networks, taking into account the distribution shift during training. 
%~\cite{duan2021risk} acknowledges the variance term in empirical Bellman loss and provides the generalization error in the single sample regime under a function space with bounded functions. 
%~\cite{nguyen2021sample} also recognizes the variance term, utilizing uniform data convergence to overcome it and obtain a generalization error under bounded multi-layer fully connected neural networks.
%~\cite{yang2020function} provides a generalization error for the LSVI algorithm with an Upper Confidence Bound (UCB,~\cite{sutton2018reinforcement}) sample mechanism.
%~\cite{wang2020provably} further incorporates an effective sampling technique~\cite{langberg2010universal, feldman2011unified, feldman2020turning} into the UCB mechanism and obtains a generalization bound under a general function space with bounded functions. 

Our work aligns closely with a significant body of research that concentrates on the generalization error and model capacity of deep neural networks. The field is particularly interested in the analysis of generalization error under various scenarios. Arora \emph{et al.}~\cite{arora2018stronger} demonstrated that a compression approach could improve the generalization error for deep neural networks. Bartlett \emph{et al.}~\cite{bartlett2017spectrally} presented a generalization error for multi-class classification problems, leveraging the size of the margin. Barron \emph{et al.}~\cite{barron2018approximation} provided a generalization error for under-parameterized neural networks. The approximation capabilities of over-parameterized neural networks across different function spaces were analyzed by Allen \emph{et al.}~\cite{allen2019learning}. Ma \emph{et al.}~\cite{ma2018priori} provided a priori estimates of two-layer neural networks under a regression setting. In another work, Ma \emph{et al.}~\cite{ma2019priori} provided a priori estimates of the residual network, and we adopted some of their analysis techniques in our work. The discussion also extends to model capacity related to norms. Zheng \emph{et al.}~\cite{zheng2019capacity} proposed a Basis-path Norm and derived a generalization error based on this norm. Neyshabur \emph{et al.}~\cite{neyshabur2015norm} provided the sample complexity for multi-layer neural networks under certain norm constraints.
%~\cite{arora2018stronger} demonstrated that a compression approach can enhance the generalization error of a deep neural network.
%~\cite{allen2019learning} analyzed the approximation capabilities of over-parameterized neural networks across different function spaces.
%~\cite{barron2018approximation} provided a generalization error for under-parameterized neural networks. 
%~\cite{bartlett2017spectrally} presented a generalization error for multi-class classification problems, utilizing the size of the margin.
%~\cite{neyshabur2015norm} provided a sample complexity for multi-layer neural networks under certain norm constraints.
%~\cite{zheng2019capacity} proposed a Basis-path Norm and constructed a generalization error based on it.
%~\cite{ma2018priori} provided an \emph{a priori} estimation of a two-layer neural network under a regression setting.
%~\cite{ma2019priori} provided an \emph{a priori} estimation of the residual network, and we have adopted some of the analysis techniques from this work.

\section{Preliminaries}

\subsection{Continuous-time reinforcement learning and discretization}

In this section, we present the notations and definitions of Markov Decision Processes (MDPs). A continuous-time deterministic discounted MDP is a tuple $(\gamma, S, A, r, g, \pi)$, where $\gamma \in (0,1)$ represents the discount factor, $S$ denotes the state space, $A$ signifies the action space, $r: S \times A \to \mathbb{R}$ is the reward function, $g: S \times A \to S$ is the deterministic transition function and $\pi: S \to A$ is a deterministic policy. Specifically, the action space $A = \{a_1, a_2, ..., a_{|A|}\}$ has a finite cardinality $|A|$. In this work, we assume $ S \subseteq [0,1]^d$, which can be easily adapted to a general compact domain in $\mathbb{R}^d$. Furthermore, we assume for any $(s,a) \in S \times A$, $|r(s,a)| \leq 1$. This can be easily extended to a bounded reward function setting.

For a given policy $\pi$ and a state $s$, the continuous-time state value function is defined as 
\begin{equation*}
    V^{\pi}(s) = \int_{0}^{\infty} e^{-\gamma t} r\left( s(t),\pi(s(t))\right) dt.
\end{equation*}
The optimal state value function is then defined as $V^{*}(s) = \max_{\pi} V^{\pi}(s)$.

\begin{remark}
    In real-world problems, given a compact set as the state space, the agent may reach the boundary. We assume that when the agent reaches the boundary, it will return to its previous position. Thus, the agent will stay in the compact domain.
\end{remark}

To apply reinforcement learning methods to solve this problem, we take a discretization. We define a discrete-time transition function, denoted as $\tilde{g}: S \times A \times \mathbb{R} \to S $, and $\Delta t \geq 0$ is the time step of discretization, where $\Delta t = 0$ represents an extreme case. With a slight misuse of notation, the state value function given a policy $\pi$ and a state $s$ is defined as 
\begin{equation*}
    V^{\pi}(s) = \sum_{i=0}^{\infty} \gamma^i r(s_i,\pi(s_i)),
\end{equation*}
where $s_0 = s$ and $s_{i+1} = \tilde{g}(s_{i}, \pi(s_{i}), \Delta t)$. The correponding state-action value function is
\begin{equation*}
    Q^{\pi}(s, a) = r(s_0,a) + \sum_{i=1}^{\infty} \gamma^i r(s_i,\pi(s_i)), \,\,\, V^{\pi}(s) = \max_{a \in A} Q^{\pi}(s,a).
\end{equation*}
The optimal state value function is defined as $V^{*}(s) = \max_{\pi} V^{\pi}$, and $Q^{*}(s, a) = \max_{\pi} Q^{\pi}$.

In this work, we primarily discuss the \emph{a priori} generalization error bound of the Bellman optimal loss, which is defined using the Bellman optimal equation. Given any $s \in S, a\in A, s' = \tilde{g}(s,a,\Delta t) \in S$, the \textbf{Bellman optimal equation} can be represented as 
\begin{equation}\label{eq::BellmanEq}
    Q(s,a) = r(s,a) + \gamma \max_{a' \in A} Q(s',a').
\end{equation}
$Q^*$ is the solution to this equation.

\subsection{Residual network and Barron space} \label{sec::barronSpace}
In this section, we present the residual network that has a skip connection at each layer:
\begin{equation}\label{eq::Def4RN}
    \begin{aligned}
         f(x,\theta) & =u^{\T}h^{[L]}, h^{[0]}=Vx,\ g^{[l]}=\sigma(W^{[l]}h^{[l-1]}),\ \\ h^{[l]} & =h^{[l-1]}+U^{[l]}g^{[l]}\text{ for }\ l \in \{1,\ldots,L\}.
    \end{aligned}
\end{equation}
Here, $x$ is an element of a compact set $ \fX \subset \sR^d$, $V\in \sR^{D\times d}$, $W^{[l]}\in \sR^{m\times D}$, $U^{[l]}\in \sR^{D\times m}$ for $l\in \{1,\ldots, L\}$, and $u\in \sR^D$. The collection of all parameters is denoted by $\theta=\operatorname{vec}\{u, V, \{W^{[l]}\}_{l=1}^L, \{U^{[l]}\}_{l=1}^L\}$, and $\sigma$ represents the ReLU function. The dimensions of the input, the width of the residual network, the width of the skip connection, and the depth of the neural network are represented by $d$, $m$, $D$, and $L$, respectively. Generally, the quadruple $(d,m,D,L)$ is used to describe the \textbf{size} of a residual network. Additionally, the function space for the residual networks is defined as:
\begin{equation}
    \hat{\fF}=\left\{f(\cdot, \theta)\colon\, f(\cdot, \theta) \text{ is a residual network as defined in~\eqref{eq::Def4RN}}\right\}.\label{eq::Fhat}
\end{equation}

In order to control the estimate, we introduce weighted path norm~\cite{ma2019priori}.
\begin{defi}[weighted path norm]\label{def::WeighPath}
    The \textbf{weighted path norm} of parameter $\theta$, real-valued residual network $f(\cdot, \theta)\in \hat{\fF}$ is 
    \begin{equation}\label{eq::WeighPath}
        \norm{\theta}_{\fP}=\Norm{\Abs{u^{\T}}(I+3\Abs{U^{[L]}}\Abs{W^{[L]}})\ldots(I+3\Abs{U^{[1]}}\Abs{W^{[1]}})\Abs{V}}_1,
    \end{equation}
    where the absolute value of a matrix or a vector means taking absolute value entry-wisely.
    We also define the \textbf{weight path norm} for real valued residual network
    \begin{equation*}
    \norm{f(\cdot, \theta)}_{\fP}=\norm{\theta}_{\fP}.
    \end{equation*}
    
\end{defi}

While the previous definition applies to real-valued functions, the state-value action function we want to approximate is a vector-valued function, necessitating an extension. We begin by defining the function space for vector-valued functions. Given a finite set $\mathbb{I}$ of size $\Abs{\mathbb{I}}$, it can be defined as:
\begin{equation}
    \fF=\left\{\{f(\cdot, \theta(i))\}_{i\in \sI}\colon\,f(\cdot,\theta(i))\in\hat{\fF}\text{ for all } i \in \sI \right\}.\label{eq::fF}
\end{equation}
Here, $\Theta=\operatorname{vec}\{\theta(i)\}_{i\in A}$, and $\theta(i)=\operatorname{vec}\{u_i, V_i, \{W^{[l]}_i\}_{l=1}^L, \{U^{[l]}_i\}_{l=1}^L\}$. The parameters of the residual network $f(\cdot, \theta(i))$ are represented by $\theta(i)$, as defined in \eqref{eq::Def4RN}. For each $i\in \sI$, $f(\cdot,\theta(i))$ is a function defined on $\fX$. Consequently, we have a total of $\abs{\sI}$ functions rather than a single one. Nonetheless, throughout the paper, we still consider $f$ as a single function on $ \fX \times \sI$. For instance, when continuity is required, we write $f\in C( \fX \times \sI)$. Based on this definition, we define the weighted path norm for vector-valued function as follow.

\begin{defi}[weighted path norm for vector-valued function]\label{def::PathVec}
    Given a finite set $\mathbb{I}$ with size $\Abs{\mathbb{I}}$, the \textbf{weighted path norm} of a $\Abs{\mathbb{I}}$-vector-valued residual network $f\in \fF$ with $\Theta=\{\theta(i)\}_{i \in \mathbb{I}}$ is
    \begin{equation*}
        \norm{\Theta}_{\fP}=\sum_{i \in \mathbb{I}}\norm{f(\cdot,\theta(i))}_{\fP}=\sum_{i \in \mathbb{I}}\norm{\theta(i)}_{\fP}.
    \end{equation*}
\end{defi}

In this paper, we consider the reward function in the Barron space, which is capable of addressing a majority of scenarios in real-world problems. Furthermore, as the reward function depends on the action, it is also a vector-valued function. In the subsequent definition, we initially introduce the Barron space for real-valued functions and subsequently extend it to vector-valued functions.

\begin{defi}[Barron space]\label{def::Barron}
    Given an activation function $\sigma: \sR \to \sR$, a real-valued function $f(x)$ defined on a compact domain $ \fX \subset \sR^d$ belongs to \textbf{Barron space}, that is, $f\in \fB(\fX)$ if and only if it can be written in the form
    \begin{equation*}
        f(y)=\mathbb{E}_{(u,w)\sim \rho}u\sigma (w\cdot x)\ \text{ for all}\ x \in  \fX,
    \end{equation*}
    where  $\rho$ is a probability distribution  over $\sR^{d+1}$.
    The \textbf{Barron norm} is defined as 
    \begin{equation*}
        \norm{f(\cdot)}_{\fB}=\inf_{\rho}\left(\mathbb{E}_{(u,w)\sim \rho}\abs{u}\norm{w}^2_1\right)^{1/2}.
    \end{equation*}
\end{defi}

\begin{defi}[Barron space for vector-valued function]\label{def::BarronVec}
    Given an activation function $\sigma: \sR \to \sR$, a finite set $\mathbb{I}$ with size $\Abs{\mathbb{I}}$ and a compact set $\fX \subset \sR^d$, a $|\sI|$-vector-valued function $f(\cdot,\cdot)$ defined on a domain $\fX \times \sI$ belongs to Barron space, i.e. $f\in \fB(\fX \times \sI)$ if and only if for each $i \in \sI$ it can be written in the form
    \begin{equation*}
        f(x,i)=\mathbb{E}_{(u_i,w_i)\sim \rho_i}u_i\sigma (w_i^{\T} x)\ \text{ for all}\ x\in  \fX,
    \end{equation*}
    where for each $i \in \sI$, $\rho_{i}$ is a probability distribution over $\sR^{d+1}$. The corresponding \textbf{Barron norm} is defined as
    \begin{equation*}
        \norm{f(\cdot,\cdot)}_{\fB}=\left(\sum_{i \in \sI}\norm{f(\cdot,i)}_{\fB}^2\right)^{1/2}.
    \end{equation*}
    We use $\norm{f}_{\fB}$ to replace $\norm{f(\cdot,\cdot)}_{\fB}$ for short notation.
\end{defi}

\begin{remark}
    From the Definition \ref{def::BarronVec}, we have the following relation: $\sum_{i \in \sI}\norm{f(\cdot,i)}_{\fB}^2 = \norm{f}_{\fB}^2$ and $|\sI|^{1/2} \norm{f}_{\fB} \geq \sum_{i \in \sI}\norm{f(\cdot,i)}_{\fB}$ because of Jensen's inequality.
\end{remark}

\subsection{Rademacher complexity}
The Rademacher complexity serves as a fundamental tool for generalization analysis. In this section, we will present the basic definition and some significant results. For the sake of completeness, they are enumerated below.

\begin{defi}[Rademacher complexity of a function class $\fF$]
    Given a set $\fS=$ $\left\{z_1, \ldots, z_n\right\}$ sample from distribution $P$, $\fS \sim P^n$, and a class $\fF$ of real-valued functions $\fZ \to \mathbb{R}$, the \textbf{Rademacher complexity} of $\fF$ on $\fZ$ is defined as
    \begin{equation*}
        \operatorname{Rad}_{\fS}(\fF)=\frac{1}{n} \mathbb{E}_\tau\left[\sup _{f \in \fF} \sum_{i=1}^n \tau_i f\left(z_i\right)\right],
    \end{equation*}
    and the \textbf{Rademacher complexity} is define as
    $\operatorname{Rad}(\fF)=\mathbb{E}_{\fS}\left[ \operatorname{Rad}_{\fS}(\fF) \right],$
    where $\tau_1, \ldots, \tau_n$ are independent random variables drawn from the Rademacher distribution, i.e., $\mathbb{P}\left(\tau_i=+1\right)=\mathbb{P}\left(\tau_i=-1\right)=\frac{1}{2}$ for $i=1, \ldots, n$.
\end{defi}

\begin{lem}[contraction lemma~\cite{ma2019priori}]\label{lem::RadContrac}
    Suppose that $\psi_i: \mathbb{R} \rightarrow \mathbb{R}$ is a $C$-Lipschitz function for each $i \in \{1,\ldots,n\}$. For any $\boldsymbol{y} \in \mathbb{R}^n$, let $\psi(\boldsymbol{y})=\left(\psi_1\left(y_1\right), \cdots, \psi_n\left(y_n\right)\right)^{\top}$. For an arbitrary set of vector functions $\fF$ of length $n$ on an arbitrary domain $ S$ and an arbitrary choice of samples $\fS=\left\{z_1, \ldots, z_n\right\}$ from distribution $P$, $\fS \sim P^n$, we have
    \begin{equation*}
        \operatorname{Rad}_S(\psi \circ \fF) \leq C \operatorname{Rad}_S(\fF), \text{ and }  \operatorname{Rad}(\psi \circ \fF) \leq C \operatorname{Rad}(\fF).
    \end{equation*}
\end{lem}

\begin{thm}[two-sided Rademacher complexity and generalization gap~\cite{shalev2014understanding}]\label{thm::RadeandGenErr}
    Suppose that $f$ 's in $\fF$ are non-negative and uniformly bounded, i.e., for any $f \in \fF$ and any $z\in  \fZ, 0 \leq f(z) \leq B$. Then for any $\delta \in(0,1)$, with probability at least $1-\delta$ over the choice of $n$ i.i.d. random samples from distribution $P$, which donate as $\fS=\left\{z_1, \ldots, z_n\right\}$ from distribution $P$, $\fS \sim P^n$, we have
    $$
    \begin{aligned}
    & \sup _{f \in \fF}\Abs{\frac{1}{n} \sum_{i=1}^n f\left(z_i\right)-\mathbb{E}_{z} f(z)} \leq 2 \operatorname{Rad}(\fF)+B \sqrt{\frac{\ln (2 / \delta)}{2 n}}.
    \end{aligned}
    $$    
\end{thm}

% \begin{lem}
%     For any RN $f(s,\theta_1)$ and $f(s,\theta_2)$ with activation $\sigma=\operatorname{ReLU}$, either $f_{\Abs{\theta_2}}(s)=f_{\Abs{\theta_1}}(s)$ or $f_{\Abs{\theta_2}}(s)=\sigma(f_{\Abs{\theta_1}}(s))$ implies that $\norm{\theta_2}_{\fP}=\norm{\theta_1}_{\fP}$, where $\Abs{\theta}$ means taking the absolute values of all the entries of the vector
% \end{lem}
\subsection{Bellman optimal loss minimization with neural network approximation}
\label{sec::originlProb}
Let's start to establish the context for the problem and methods. Initially, we redefine the function spaces \eqref{eq::Fhat} and \eqref{eq::fF} from Section \ref{sec::barronSpace}, utilizing the state space $S$ and action space $A$. Subsequently, we extend our definition to the following function spaces:
\begin{equation*}
    \begin{aligned}
        \hat{\fF}_M&=\left\{ f(\cdot, \theta) \in \hat{\fF} \colon\, \norm{\theta}_{\fP}\le M \right\}, 
        \fF_M =\Bigl\{f(\cdot, \theta(\cdot)) \in \fF\colon\, \norm{\Theta}_{\fP}\le M \Bigr\}, \\
        \fZ^{\max}&=\left\{\max_{a\in A}f(\cdot, \theta(a))\colon\, f \in \fF\right\},\ 
        \fZ^{\max}_M =\left\{\max_{a\in A}f(\cdot, \theta(a))\colon\,f \in \fF_M\right\}. \\
        \fG &=\left\{f(\cdot, \theta(\cdot))-\gamma\max_{a\in A}f(\cdot, \theta(a))\colon\, f \in \fF\right\}, \\
        \fG_{M} &=\left\{f(\cdot, \theta(\cdot))-\gamma\max_{a\in A}f(\cdot, \theta(a))\colon\,f \in \fF_M\right\}.
    \end{aligned}
\end{equation*}
According to the given definition, if $f(\cdot, \theta(\cdot)) \in \fF_{M}$, then $f(\cdot, \theta(a)) \in \hat{\fF}_{M}$ for all $a \in A$.

From the Bellman equation given in \eqref{eq::BellmanEq}, we define the \textbf{Bellman optimal loss} as 
\begin{equation*}
    \fR_{\fD}(\Theta)=\frac{1}{2}\mathbb{E}_{s \sim \fD, a \sim \fU} \left(f(s, \theta(a))-r(s,a)-\gamma\max_{a'\in A}f(s', \theta(a'))\right)^2, 
\end{equation*}
where $\fD$ and $\fU$ are uniform distributions over $S$ and $A$ respectively, $f \in \fF$, $\gamma\in (0,1)$, $r(\cdot,a)\in C( S)$ for all $a\in A$ and $s'=\tilde{g}(s,a,\Delta t)$. Finding $f^*$ that minimizes the Bellman optimal loss refers as \textbf{Bellman optimal loss minimization problem}. 

For sample $\fS=\{s_i, a_i, s_i', r\}_{i=1}^{n}$, the corresponding \textbf{empirical loss function} is
\begin{equation*}
    \fR_{\fS}(\Theta)=\frac{1}{2n}\sum_{i=1}^{n}\left[f(s_i,\theta(a_i))-r(s_i,a_i)-\gamma \max_{a'\in A}f(s'_i,\theta(a'))\right]^2.
\end{equation*}

In this study, we optimize a regularized empirical loss to find the solution for the Bellman optimal loss. The optimization problem is defined as follows:
\begin{equation}\label{eq::RegLoss}
    \min_{f \in \fF} \left \{ \fR_{\fS}(\Theta) + \lambda \Lambda(\norm{\Theta}_{\fP}) \right\},
\end{equation}
where $\Lambda(\cdot)$ represents the regularization function and $\lambda$ is the regularization constant. The regularization technique is widely used and beneficial in the training of neural networks. It can prevent overfitting~\cite{santos2022avoiding}, enhance the generalization ability~\cite{yoshida2017spectral, foret2020sharpness}, assist in feature selection~\cite{zhang2019feature}, and increase the interpretability of the model~\cite{wu2017improving, wu2021optimizing}. 

\subsection{Bellman effective loss minimization problem}
In Section \ref{sec::barronSpace}, we introduce a discrete transition function $\tilde{g}$, which is a key concept in our modeling. We now make some basic assumptions for it.
\begin{assum}[properties of transition function]\label{assum::SemiGroup}
    \begin{enumerate}
        \item Assume that $\tilde{g}$ is a semi-group, i.e., $\tilde{g}(\tilde{g}(s,a,t),a,\tilde{t})=\tilde{g}(s,a,t+\tilde{t})$ for all $t,\tilde{t} \geq 0$; $\tilde{g}(s,a,0)=s$.
        \item Assume $\tilde{g}$ is Lipschitz in $t$, that is for any $s \in S, a\in A$ and $t\in [0,+\infty)$, there is constant $C_T>0$ such that for all $\Delta t \in (0, 1)$, $\norm{\tilde{g}(s,a, t)-\tilde{g}(s,a, t+\Delta t)}_{\infty}\le C_T \Delta t$.
    \end{enumerate}
\end{assum}

From this assumption, we have $s' = \tilde{g}(s,a, \Delta t) = s + O(\Delta t)$. If there is a continuous function $f(s,a)$ satisfies equation \eqref{eq::BellmanEq}, we have
\begin{equation*}
    \begin{aligned}
        f(s,a)
        &=r(s,a)+\gamma \max_{a'\in A}f(s+O(\Delta t),a')\\
        &=r(s,a)+\gamma\max_{a'\in A}f(s,a')+O(\Delta t)\frac{\diff{f}}{\diff{s}}(s,a').
    \end{aligned}
\end{equation*}
This relation represents the smooth policy property. Additionally, if $\Delta t \to 0$, the third term will be negligible. So we can consider an extreme case, where we set $\Delta t$ to zero, we obtain a \textbf{Bellman effective equation} as
\begin{equation}\label{eq::ModEq}
    f^*(s,a)=r(s,a)+\max_{a'\in A}f^*(s,a').
\end{equation}
Then we define the \textbf{Bellman effective loss} following the equation \eqref{eq::ModEq} with uniform distribution $\fD$ and $\fU$ as
\begin{equation*}
    \tilde{\fR}_{\fD}(\Theta)=\frac{1}{2}\mathbb{E}_{s \sim \fD, a \sim \fU} \left(f(s, \theta(a))-r(s,a)-\gamma\max_{a'\in A}f(s, \theta(a'))\right)^2.
\end{equation*}
Finally, we define the \textbf{Bellman effective loss minimization problem} as finding the function $f^*$ that minimizes the Bellman effective loss. We donate this $f^*$ as \textbf{effective solution}. 
    
\section{Main results}\label{sec::MainResults}

\subsection{Main theorem}
We first introduce a mild assumption regarding the reward function.

\begin{assum}[reward function is Barron]\label{assump::RewardFunctionBarron}
    Suppose that $r(\cdot,a)\in \fB( S\times A)$ for each $a \in A$ and that for all $(s,a)\in S \times A$, $|r(s,a)| \leq 1$.
\end{assum}

The main theorem provides an \emph{a priori} estimates of the Bellman optimal loss with residual network approximation and explicit regularization. 

\begin{thm}[\emph{a priori} generalization error bound for Bellman optimal loss]\label{thm::AprioriGenErrBoun}
    Suppose that Assumption~ \ref{assum::SemiGroup} and \ref{assump::RewardFunctionBarron} hold, let 
    \begin{equation*}
        \hat{\Theta}=\arg\min_{\Theta}\fJ_{S,\lambda}(\Theta):=\fR_{\fS}(\Theta)+\frac{\lambda}{\sqrt{n}}\norm{\Theta}_{\fP}^2\ln (4(\norm{\Theta}_{\fP}+1)),
    \end{equation*}
    with a residual network $f(s, \hat{\theta}(a))$ of size $(d,(6^{\alpha}+1)m,D_0,\alpha+1)$ and $\hat{\Theta}=\{\hat{\theta}(a)\}_{a\in A}$. 
    Then for any $\gamma \in (0, 1)$ and $\delta\in (0,1)$, with probability at least $1-\delta$ over the choice of a i.i.d. sample $\fS=\{s_i, a_i, s_i', r\}_{i=1}^{n}$ and $\lambda > 72|A|^3\sqrt{2\ln (2d)} + 9\ln(\Abs{A}/\delta) + 18$, we have 
    \begin{equation} \label{eq::ApriorGenResult}
        \begin{aligned}
            \fR_{\fD}(\hat{\Theta}) 
            & \leq \frac{\operatorname{Poly}(\Abs{A})}{(1-\gamma)^2} \left( \frac{1}{m} + (\Delta t)^2  \right) \norm{r}_{\fB}^2 + \frac{(\lambda + 1)\operatorname{Poly}(\Abs{A}, \ln d, \Abs{\ln \delta})}{(1-\gamma)^2 \sqrt{n}} \norm{r}_{\fB}^2 \left( \norm{r}_{\fB} + 1\right)
        \end{aligned}.
    \end{equation}
\end{thm}

\begin{remark}
    In this section, we employ $\mathrm{Poly}(\cdot)$ to streamline the expression. Here, $\mathrm{Poly}(\cdot)$ denotes a specific polynomial function that may vary from one line to another. The original expression is presented in Section \ref{sec::ProofThm}.
\end{remark}

\begin{remark}
    Selecting an appropriate $\Delta t$ involves a trade-off. In \eqref{eq::ApriorGenResult}, $\Delta t$ governs the sample complexity, with a smaller $\Delta t$ resulting in a smaller error bound. One can choose $\Delta t \leq O(\max \{ 1/m^{1/2}, 1/n^{1/4} \})$, where the term contains $\Delta t$ will not be the dominant term in the bound. Moreover, the choice of $\Delta t$ impacts both computational efficiency and control accuracy in the actual problem. A larger $\Delta t$ enhances computational efficiency at the expense of control accuracy, while a smaller $\Delta t$ improves control accuracy but reduces computational efficiency. Balancing these two aspects, one can select $\Delta t = O(\max \{ 1/m^{1/2}, 1/n^{1/4} \})$.
\end{remark}

\begin{remark}
    We would like to emphasize that our estimate is nearly optimal in terms of the sample size and the model size. For the first term, we derive $m = (NL) / C_{|A|}$ from Theorem \ref{thm::apprErrorOrigin}, where $N$ represents the total width of the neural network, $L$ denotes the depth of the neural network, and $C_{|A|}$ is a finite constant associated with the size of the action space. The convergence rate is $O(1/(LN))$, which aligns with the rate in the universal approximation theory for shallow networks~\cite{barron1993universal}. The second term illustrates the rate with respect to the sample size as $O(1/\sqrt{n})$, which corresponds to the Monte Carlo rate derived from classical estimates of the generalization gap. Also, $\Delta t$ is a small constant. Moreover, the bound dependent on the polynomial of the action space size, which is a finite number and thus, does not contribute to the curse of dimensionality.
\end{remark}

\subsection{Proof sketch}

The structure of the proof is illustrated in Figure \ref{fig:theorem_diagram}. The proof process can be divided into two parts: the Bellman effective loss minimization problem and the Bellman optimal loss minimization problem. The \emph{a priori} generalization error of the Bellman optimal loss is obtained by combining its corresponding approximation error and the a posterior generalization error. The non-trivial part of the proof is estimating the approximation error of the Bellman optimal loss, which we use two transfers to achieve it. Fist of all, since the Bellman effective loss has an explicit solution, we can construct a supervised loss that fits this solution and estimate the corresponding approximation error. Also, the approximation error of the Bellman effective loss can be bounded by the combination of the approximation error of the supervised loss, then we can estimate the approximation error of the Bellman effective loss, completing the first transfer. Afterward, since the transition function has the Lipschitz property, the approximation error of the Bellman optimal loss can be controlled by the approximation error of the Bellman effective loss plus a small quantity, completing the second transfer.

% The structure of the proof is illustrated in Figure \ref{fig:theorem_diagram}. The proof process can be divided into two parts.
% In the first part, we estimate the approximation error within the context of the effective problem. In the second part, we transfer to the original problem and derive its \emph{a priori} generalization error. In effective problem, we provide the approximation error in terms of supervised learning loss and Bellman effective loss. After we shift to the original problem, we present the final \emph{a priori} generalization error for the Bellman optimal loss.

% Throughout the proof, there are two transitions of the loss function: from supervised loss to Bellman effective loss, and from Bellman effective loss to Bellman optimal loss. Initially, we establish the approximation error bound for a supervised loss of the effective solution. Because the solution of the Bellman effective loss has an explicit form, we employ a regression method to fit the solution and determine the approximation error. Subsequently, the Bellman effective loss is bounded by a combination of supervised loss and we are able to estimate the approximation error of the Bellman effective loss, which accomplishes the transition from supervised loss to Bellman effective loss. Following this, thanks to the Lipschitz property of the transition function $\tilde{g}$, the Bellman optimal loss is bounded by a combination of the Bellman effective loss and an additional term. This allows us to transition from the Bellman effective loss to the Bellman optimal loss.

\begin{figure}[htb]
\centering
% include first image
\includegraphics[width=1\linewidth]{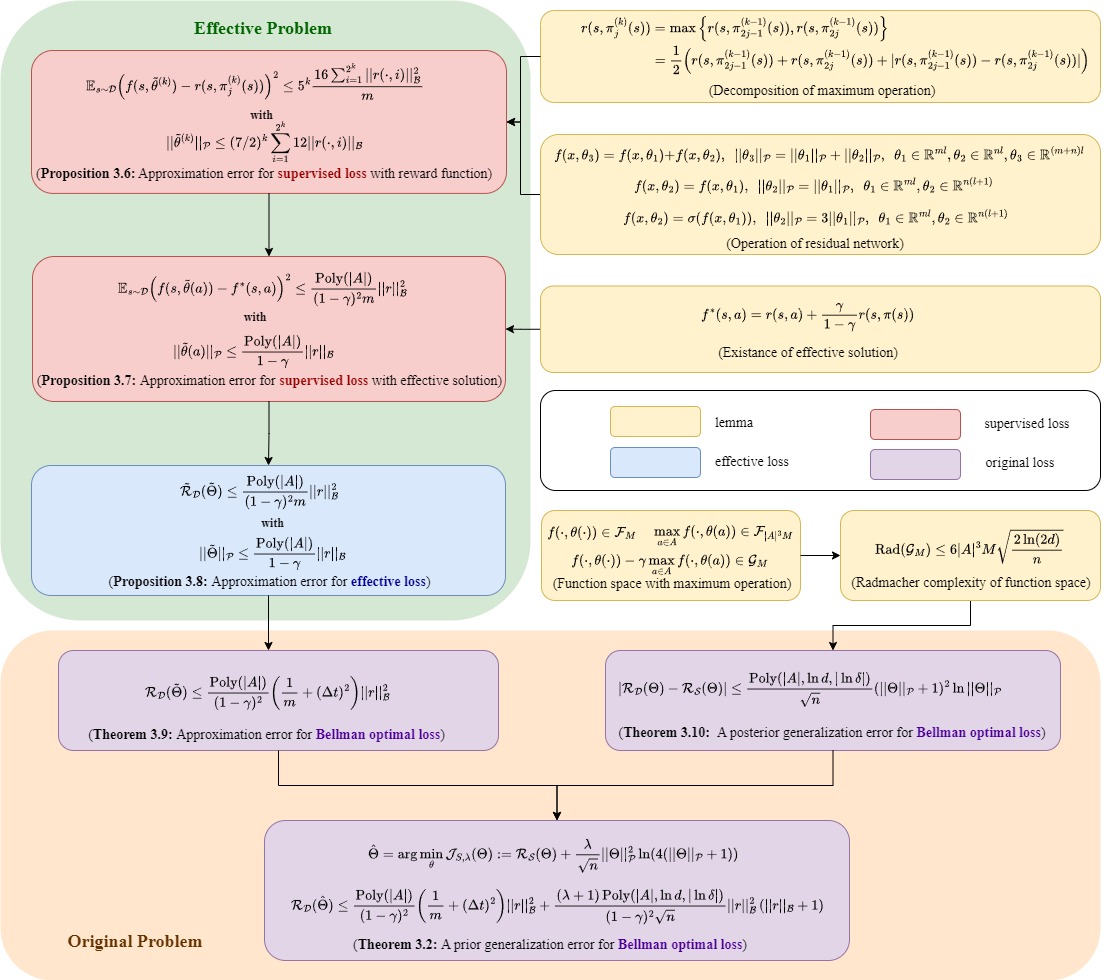}
\caption{Sketch of proof for Theorem \ref{thm::AprioriGenErrBoun}}
\label{fig:theorem_diagram}
\end{figure}

The proof starts from the Bellman effective loss minimization problem. Lemma \ref{lem::ExistofMod} demonstrates that this problem possesses an explicit solution given by:
\begin{equation}\label{eq::EffectiveSol}
        f^*(s,a)=r(s,a)+\frac{\gamma}{1-\gamma}\max_{a' \in A} r(s,a').
\end{equation}
This effective solution can be expressed by reward functions.

We now proceed to estimate the error bound for the supervised loss of the effective solution. However, estimating the error bound between $f(\cdot, \theta)$ and $\max_{a' \in A} r(s,a')$ poses a challenge due to the maximum operation. Interestingly, we observe that the maximum operation over a set can be transformed into a maximum operation over two elements in a binary tree structure, as depicted in Figure \ref{fig:tree_structure_of_policy}. Besides, the maximum of two elements can be expressed as:
\begin{equation} \label{eq::maxOpera2}
    \max\{a,b\} = \frac{1}{2}(a + b + |b - a|) = \frac{1}{2}(a + b + \sigma(a - b) + \sigma(b-a)),
\end{equation}
where $\sigma$ represents the ReLU function. Consequently, we can leverage the structure of the neural network to replace the maximum operation.

Building upon the aforementioned idea, we can construct a binary tree for the reward function. We set $A=\{1,\ldots,2^{\alpha}\}$ without loss of generality. Even if the size of the action space is less than $2^{\alpha}$, we can still construct a binary tree and follow the same methodology. We define $A^{(k)}=\{\pi^{(k)}_{1}(s),\ldots, \pi^{(k)}_{2^{\alpha-k}}(s)\}$ for $k\in \{1,\ldots,\alpha\}$, and $\pi^{(0)}_{j}(s)=j,$ for $j\in A$, $s\in S$. Also, for $j\in \{1,\ldots, 2^{\alpha-k}\}$ and $k\in\{1,\ldots, \alpha\}$, we define:
\begin{equation}\label{eq::maxOperationPolicy}
        \pi^{(k)}_{j}(s)=\arg\max_{\{\pi^{(k-1)}_{2j-1}(s),\pi^{(k-1)}_{2j}(s)\}}\{r(s,\pi^{(k-1)}_{2j-1}(s)),r(s,\pi^{(k-1)}_{2j}(s))\}.
\end{equation}
$\pi^{(k)}_{j}(s)$ can also be represented in the format of \eqref{eq::maxOpera2}, as given in Lemma \ref{lem::FuncDecomp}. The relationship between $\pi^{(k)}_{j}(s)$ shows in Figure \ref{fig:tree_structure_of_policy}, which is a binary tree. Note that $\pi^{(\alpha)}_{1}(s) = \pi(s)=\arg\max_{a\in A}r(s,a)$.

\begin{figure}[htb]
\centering
% include first image
\includegraphics[width=0.6\linewidth]{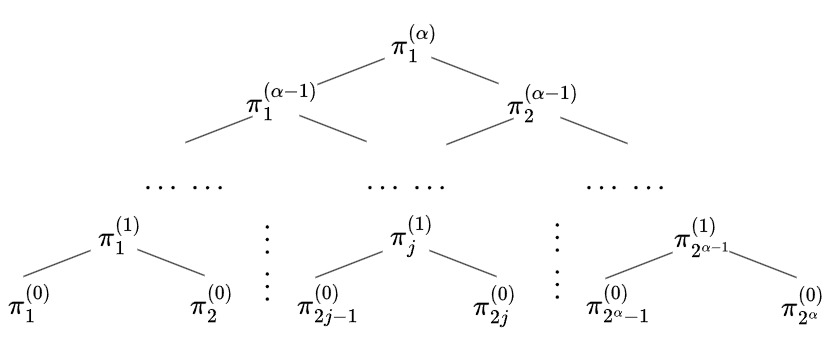}
\caption{The relation between $\pi_j^{(k)}$, which is a binary tree.}
\label{fig:tree_structure_of_policy}
\end{figure}

From Lemmas \ref{lem::CombRN}, \ref{lem::DeepenRN}, and \ref{lem::DeepenRN4Activate}, we are able to construct an appropriate residual network to substitute the maximum operation. Consequently, we can determine the approximation error for the supervised loss with reward function and maximum operation, given a suitable residual network.
\begin{prop}[approximation error for supervised loss with reward function] \label{prop::ApproxErr4Max}
    For given integer $k\in\{0,1,\ldots, \alpha\}$, $j\in \{1,\ldots, 2^{\alpha-k}\}$, there is integer $D\in \sN^+$ and an residual network $f(s, \tilde{\theta}^{k}_{j})$ with size $(d,6^{k}m,D,k+1)$ such that
    \begin{equation*}
        \mathbb{E}_{s \sim  \fD}\left(f(s, \tilde{\theta}^{k}_{j})-r(s,\pi^{(k)}_{j}(s))\right)^2 \le 5^{k}\frac{16\sum_{i=(j-1)2^k+1}^{j2^{k}}\norm{r(\cdot,i)}_{\fB}^2}{m}
    \end{equation*}
    with $\norm{\tilde{\theta}^{(k)}_j}_{\fP}\le (7/2)^{k}\sum_{i=(j-1)2^{k}+1}^{j2^{k}}12\norm{r(\cdot,i)}_{\fB}.$
\end{prop}

By integrating the last proposition with the effective solution provided in \eqref{eq::EffectiveSol}, we are able to estimate the approximation error for supervised loss with effective solution.
\begin{coro}[approximation error for supervised loss with effective solution] \label{coro::ApproxErr4Sing}
    Suppose that Assumption~\ref{assump::RewardFunctionBarron} holds. Then for all $a\in A$ and $\gamma\in [0,1)$, there is integer $D_0\in \sN^+$ and an residual network $f(s, \tilde{\theta}(a))$ with size $(d,(6^{\alpha}+1)m,D_0,\alpha+1)$ such that
    \begin{equation*}
        \mathbb{E}_{s \sim  \fD}\left(f(s, \tilde{\theta}(a))-f^*(s,a)\right)^2 \le \frac{\operatorname{Poly}(\Abs{A})}{(1-\gamma)^2 m}\norm{r}_{\fB}^2
    \end{equation*}
    with $\norm{\tilde{\theta}(a)}_{\fP}\le \frac{\operatorname{Poly}(\Abs{A})}{1-\gamma}\norm{r}_{\fB}.$    
\end{coro}
\begin{remark}
    Notice that $6^\alpha < \Abs{A}^3$, so network width is less than $(\Abs{A}^3+1)m$.
\end{remark}

The the Bellman effective loss is bounded by a combination of the supervised loss of the effective solution. This leads us to the following proposition.
\begin{prop}[approximation error for Bellman effective loss] \label{prop::ApproxErr4Sum}
     Suppose that Assumption~\ref{assump::RewardFunctionBarron} holds. The for all $a\in A$ and $\gamma\in (0,1)$, there are integer $D_0\in \sN^+$,  residual network $f(s, \tilde{\theta}(a))$ with size $(d,(6^{\alpha}+1)m,D_0,\alpha+1)$ and $\Tilde{\Theta}=\{\tilde{\theta}(a)\}_{a\in A}$ such that
    \begin{equation*}
        \begin{aligned}
            \tilde{\fR}_{\fD}(\Tilde{\Theta}) \le \frac{\operatorname{Poly}(\Abs{A})}{(1-\gamma)^2 m}\norm{r}_{\fB}^2
        \end{aligned}
    \end{equation*}
    with $\norm{\tilde{\Theta}}_{\fP} \le \frac{\operatorname{Poly}(\Abs{A})}{1-\gamma}\norm{r}_{\fB}.$ 
\end{prop}

Thanks to the Lipschitz property of the transition function given in Assumption \ref{assum::SemiGroup}, the Bellman optimal loss is bounded by a combination of the Bellman effective loss and an additional term associated with the time step of discretization.
\begin{thm}[approximation error for Bellman optimal loss]\label{thm::apprErrorOrigin}
    Suppose that Assumption~ \ref{assum::SemiGroup} and \ref{assump::RewardFunctionBarron} hold. Then for all $a\in A$ and $\gamma\in (0,1)$, there are integer $D_0\in \sN^+$,  residual network $f(s, \tilde{\theta}(a))$ with size $(d,(6^{\alpha}+1)m,D_0,\alpha+1)$ and $\Tilde{\Theta}=\{\tilde{\theta}(a)\}_{a\in A}$ such that
    \begin{equation*}
        \begin{aligned}
            \fR_{\fD}(\tilde{\Theta}) & \leq  \frac{\operatorname{Poly}(\Abs{A})}{(1-\gamma)^2} \left( \frac{1}{m} + (\Delta t)^2  \right) \norm{r}_{\fB}^2.
        \end{aligned}
    \end{equation*}
    with $\norm{\tilde{\Theta}}_{\fP} \leq \frac{\operatorname{Poly}(\Abs{A})}{1-\gamma}\norm{r}_{\fB}.$
\end{thm}

To obtain the \emph{a posteriori} generalization bound, it is necessary to estimate the Radema-cher complexity of the function space associated with residual networks and maximum operation, which we can employ the same methodology as in Proposition \ref{prop::ApproxErr4Max}. By integrating Lemma \ref{lem::RNSpace4Max}, Lemma \ref{lem::Rad}, and Theorem \ref{thm::RadeandGenErr}, we arrive at the following result.
\begin{thm}[a posterior generalization bound for Bellman optimal loss] \label{thm::PosterGenBoun}
    Suppose that Assumption~\ref{assump::RewardFunctionBarron} holds. For any $\gamma \in (0,1)$ and $\delta>0$, with probability $1-\delta$ over the choice of $n$ i.i.d. samples from $(\fD \times \fU)$, donated as $\fS=\{s_i, a_i, s_i', r\}_{i=1}^{n}$, and $f(s,\theta(a))$ be the residual network defined in \eqref{eq::Def4RN} with $\Theta = \{\theta(a)\}_{a \in A}$, we have that 
    \begin{equation*}
        \Abs{\fR_{\fD}(\Theta)-\fR_{\fS}(\Theta)}\le \frac{\operatorname{Poly}(\Abs{A}, \ln d, \Abs{\ln \delta})}{\sqrt{n}}(\norm{\Theta}_{\fP}+1)^2 \ln \norm{\Theta}_{\fP}.
    \end{equation*}
\end{thm}

We are now prepared to prove the final result. Let's consider the following decomposition:
\begin{equation}\label{eq::LossDecom}
    \fR_{\fD}(\hat{\Theta})=\fR_{\fD}(\tilde{\Theta})+[\fR_{\fD}(\hat{\Theta})-\fJ_{S,\lambda}(\hat{\Theta})]+[\fJ_{S,\lambda}(\hat{\Theta})-\fJ_{S,\lambda}(\tilde{\Theta})]+[\fJ_{S,\lambda}(\tilde{\Theta})-\fR_{\fD}(\tilde{\Theta})].
\end{equation}
Here, $\hat{\Theta}$ represents the optimal solution of the regularized loss function \eqref{eq::RegLoss}, and $\tilde{\Theta}$ corresponds to the residual network given in Theorem \ref{thm::apprErrorOrigin}. From Theorem \ref{thm::apprErrorOrigin}, we can derive the bound of $\fR_{\fD}(\tilde{\Theta})$. According to Theorem \ref{thm::PosterGenBoun}, both $\fR_{\fD}(\hat{\Theta})-\fJ_{S,\lambda}(\hat{\Theta})$ and $\fJ_{S,\lambda}(\tilde{\Theta})-\fR_{\fD}(\tilde{\Theta})$ are bounded with high probability. Furthermore, from the definition, $\fJ_{S,\lambda}(\hat{\Theta})-\fJ_{S,\lambda}(\tilde{\Theta}) \leq 0$. Substituting all of the above into \eqref{eq::LossDecom} yields the \emph{a priori} estimates in Theorem \ref{thm::AprioriGenErrBoun}.

\section{Proof of Theorems}\label{sec::ProofThm}

\subsection{Lemmas for the operations of residual network}

In this section, we present three straightforward lemmas associated with the construction of the maximum operation.

\begin{lem}[combination of residual network]\label{lem::CombRN}
    For two residual networks $f(x, \theta_1)$ and $f(x, \theta_2)$ with size $(d,m_1,D_1,L)$ and $(d,m_2,D_2,L)$ respectively, there exists an residual network $f(x, \theta_3)$ with size $(d,m_1+m_2,D_1+D_2,L)$ such that
    \begin{equation*}
        f(x, \theta_3)=f(x, \theta_1)+f(x, \theta_2),
    \end{equation*}
    and that $\norm{\theta_3}_{\fP}=\norm{\theta_1}_{\fP}+\norm{\theta_2}_{\fP}$.
\end{lem}
\begin{proof}
    By setting
    \begin{align*}
        f(x, \theta_1)&=u_1^{\T}h_1^{[L]}, \ h_1^{[l]}=h_1^{[l-1]}+U_1^{[l]}\sigma(W_1^{[l]}h_1^{[l-1]})\ \text{ for }\ l=\{1,\ldots,L\},\ h_1^{[0]}=V_1 x,\\
        f(x, \theta_2)&=u_2^{\T}h_2^{[L]}, \ h_2^{[l]}=h_2^{[l-1]}+U_2^{[l]}\sigma(W_2^{[l]}h_2^{[l-1]})\ \text{ for }\ l=\{1,\ldots,L\},\ h_2^{[0]}=V_2 x.
    \end{align*}
    Construct $u_3^{\T}=\{u_1^{\T}, u_2^{\T}\}\in \sR^{1\times (D_1+D_2)}$, $V_3=\left[ \begin{array}{c}
         V_1   \\
        V_2
    \end{array}\right]\in \sR^{(D_1+D_2)\times d}$, $W_3^{[l]}=\left[ \begin{array}{cc}
         W_1^{[l]} & 0  \\
          0 & W_2^{[l]}
    \end{array}\right]\in \sR^{(m_1+m_2)\times (D_1+D_2)}$ and $U_3^{[l]}=\left[ \begin{array}{cc}
         U_1^{[l]} & 0  \\
          0 & U_2^{[l]}
    \end{array}\right]\in \sR^{(D_1+D_2)\times (m_1+m_2)}$ for $l\in \{1,\ldots ,L\}$ with $\theta_3=\operatorname{vec}\{u_3, V_3, \{W_3^{[l]}\}_{l=1}^L, \{U_3^{[l]}\}_{l=1}^L\}$. It easy to verify  $f(x, \theta_3)=f(x, \theta_1)+f(x, \theta_2)$ and that $\norm{\theta_3}_{\fP}=\norm{\theta_1}_{\fP}+\norm{\theta_2}_{\fP}$ due to that $V_3 x= \left[ \begin{array}{cc}
         V_1 x\\
         V_2 x
        \end{array}\right]$ and $W_3^{[l]}$, $U_3^{[l]}$ are partitioned matrix. 
\end{proof}
\begin{lem}[deepen residual network with extra layer]\label{lem::DeepenRN}
    Given integer $L\in \sN^+$, for residual network $f(x, \theta_1)$ with size $(d,m,D,L)$ , there exists residual network $f(x, \theta_2)$ with size $(d,m,D+1,L+1)$ such that $f(x, \theta_2)=f(x, \theta_1)$ and $\norm{\theta_2}_{\fP} = \norm{\theta_1}_{\fP}$
\end{lem}
\begin{proof}
    Let $f(x, \theta_1)$ be
    \begin{equation*}
        f(s,\theta_1)=u_1^{\T}h_1^{[L]}, \ h_1^{[l]}=h_1^{[l-1]}+U_1^{[l]}\sigma(W_1^{[l]}h_1^{[l-1]})\ \text{ for }\ l=\{1,\ldots,L\},\ h_1^{[0]}=V_1 x.
    \end{equation*}
    Construct $V_2=\left[ \begin{array}{c}
         V_1   \\
         0
    \end{array}\right]\in \sR^{(D+1)\times d}$ and $W_2^{[l]}=\left[ \begin{array}{cc}
         W_1^{[l]} & 0  \\
        \end{array}\right]\in \sR^{m\times (D+1)}$ and $U_2^{[l]}=\left[ \begin{array}{cc}
         U_1^{[l]}   \\
          0 
    \end{array}\right]\in \sR^{(D+1)\times m}$ for $l\in \{1,\ldots ,L\}$, we see that $h_2^{[l]}=\left[ \begin{array}{c}
         h_1^{[l]}   \\
         0
    \end{array}\right]\in \sR^{D+1}$ for $l\in\{1,\ldots,L\}$. Then let
    \begin{equation*}
        W_2^{[L+1]}=\left[ \begin{array}{ccccc}
         0   
    \end{array}\right]\in \sR^{m\times (D+1)}, U_2^{[L+1]}=\left[ \begin{array}{ccccc}
         0 
    \end{array}\right]\in \sR^{(D+1)\times m}.
    \end{equation*}
    Hence 
    \begin{equation*}
        h_2^{[L+1]}=h_2^{[L]}+U_2^{[L+1]}\sigma(W_2^{[L+1]}h_2^{[L]})=\left[ \begin{array}{c}
         h_1^{[L]}   \\
         0
    \end{array}\right]\in \sR^{D+1}.
    \end{equation*}
    Thus by letting $u_2^{\T}=\{u_1^{\T},0\}\in \sR^{1\times (D+1)}$, we see that
    \begin{equation*}
        f(x, \theta_2)=u_2^{\T}h_2^{[L+1]}=u_1^{\T}h_1^{[L]}=f(x, \theta_1).
    \end{equation*}
    Moreover, 
    \begin{equation*}
        \Abs{u_2^{\T}}\left(I+3\Abs{U_2^{[L+1]}}\Abs{W_2^{[L+1]}}\right)=\Abs{\{u_1^{\T},0\}},
    \end{equation*}
    and this implies that $\norm{\theta_2}_{\fP}=\norm{\theta_1}_{\fP}$.
\end{proof}

\begin{lem}[deepen residual network with activation function]\label{lem::DeepenRN4Activate}
    Given integer $L\in \sN^+$, for residual network $f(x, \theta_1)$ with size $(d,m,D,L)$ , there exists residual network $f(x, \theta_2)$ with size $(d,m,D+1,L+1)$ such that $f(x, \theta_2)=\sigma(f(x, \theta_1))$ and that $\norm{\theta_2}_{\fP}=3\norm{\theta_1}_{\fP}$
\end{lem}
\begin{proof}
    Let $f(x, \theta_1)$ be
    \begin{equation*}
        f(x, \theta_1)=u_1^{\T}h_1^{[L]}, \ h_1^{[l]}=h_1^{[l-1]}+U_1^{[l]}\sigma(W_1^{[l]}h_1^{[l-1]})\ \text{ for }\ l=\{1,\ldots,L\},\ h_1^{[0]}=V_1 x.
    \end{equation*}
    Construct $V_2=\left[ \begin{array}{c}
         V_1   \\
         0
    \end{array}\right]\in \sR^{(D+1)\times d}$ and $W_2^{[l]}=\left[ \begin{array}{cc}
         W_1^{[l]} & 0  \\
        \end{array}\right]\in \sR^{m\times (D+1)}$ and $U_2^{[l]}=\left[ \begin{array}{cc}
         U_1^{[l]}   \\
          0 
    \end{array}\right]\in \sR^{(D+1)\times m}$ for $l\in \{1,\ldots ,L\}$, we see that $h_2^{[l]}=\left[ \begin{array}{c}
         h_1^{[l]}   \\
         0
    \end{array}\right]\in \sR^{D+1}$ for $l\in\{1,\ldots,L\}$. Then let
    \begin{equation*}
        W_2^{[L+1]}=\left[ \begin{array}{ccccc}
         0 & 0  \\
         u_1^{\T} & 0  
    \end{array}\right]\in \sR^{m\times (D+1)}, U_2^{[L+1]}=\left[ \begin{array}{ccccc}
         0 & \ldots & 0  \\
         0 & \ldots  & 1\\
    \end{array}\right]\in \sR^{(D+1)\times m}.
    \end{equation*}
    Hence 
    \begin{equation*}
        h_2^{[L+1]}=h_2^{[L]}+U_2^{[L+1]}\sigma(W_2^{[L+1]}h_2^{[L]})=\left[ \begin{array}{c}
         h_1^{[L]}   \\
         \sigma(u_1^{\T}h_1^{[L]})
    \end{array}\right]\in \sR^{D+1}.
    \end{equation*}
    Thus by letting $u_2^{\T}=\{0,0,\ldots,1\}\in \sR^{1\times (D+1)}$, we see that
    \begin{equation*}
        f(x, \theta_2)=u_2^{\T}h_2^{[L+1]}=\sigma(u_1^{\T}h_1^{[L]})=\sigma(f(x, \theta_1)).
    \end{equation*}
     Moreover, 
    \begin{equation*}
        \Abs{u_2^{\T}}\left(I+3\Abs{U_2^{[L+1]}}\Abs{W_2^{[L+1]}}\right)=\Abs{\{3u_1^{\T},1\}},\ \Abs{V_2} x=\left[ \begin{array}{c}
         \Abs{V_1} x   \\
         0
    \end{array}\right]
    \end{equation*}
    and this implies that $\norm{\theta_2}_{\fP}=3\norm{\theta_1}_{\fP}$.
\end{proof}

\subsection{Approximation error for the Bellman effective loss minimization problem}

We first prove two lemmas: 1) the existence lemma of effective solution and 2) the decomposition lemma for the equation \eqref{eq::maxOperationPolicy}.

\begin{lem}[existence of effective solution]\label{lem::ExistofMod}
    There exists a function $f^*(\cdot,\cdot)\in C( S\times A)$ which satisfies the Bellman effective equation \eqref{eq::ModEq} and it can be represented as
    \begin{equation*}
        f^*(s,a)=r(s,a)+\frac{\gamma}{1-\gamma}\max_{a \in A}r(s,a).
    \end{equation*}
    $f^*(\cdot,\cdot)$ is the \textbf{effective solution}.
\end{lem}
\begin{proof}
    If such $f^*$ exists, it should satisfy
    \begin{equation}\label{eq::Invari}
        \begin{aligned}
            f^*(s,a)&=r(s,a)+\gamma\max_{a'\in A}f^*(s,a');\\
            f^*(s,\Tilde{a}) &= r(s,\tilde{a})+\gamma\max_{a'\in A}f^*(s,a').
        \end{aligned}     
    \end{equation}
    Thus 
    \begin{equation*}
        f^*(s,a)-f^*(s,\Tilde{a})=r(s,a)-r(s,\Tilde{a}).
    \end{equation*}
    For $s$ fixed, we define a policy as $\pi(s)=\arg\max_{a \in A} f^*(s,a)$, we have
    \begin{equation}\label{eq::EqofMax}
        f^*(s,\pi(s))=r(s,\pi(s))+\gamma f^*(s,\pi(s)) =\frac{1}{1-\gamma}r(s,\pi(s))
    \end{equation}
    Equations \eqref{eq::Invari} and \eqref{eq::EqofMax} together lead to 
    \begin{equation}\label{eq::Solf*}
        f^*(s,a)=r(s,a)-r(s,\pi(s))+\frac{1}{1-\gamma}r(s,\pi(s))=r(s,a)+\frac{\gamma}{1-\gamma}r(s,\pi(s)).
    \end{equation}
    Conversely, it can be easily checked that $f^*$ given by equation \eqref{eq::Solf*} is a solution to the Bellman effective equation \eqref{eq::ModEq} and $\pi(s) = \arg\max_{a \in A}r(s,a)$.

    Since for each $a\in A$, $r(s,a)$ is continuous and $r(s,\pi(s))=\max_{a\in A}r(s,a)$, we have $r(s,\pi(s))$ is continuous, thus function $f^*(s,a)$ is continuous and hence the proof is completed.
\end{proof}

\begin{lem}[decomposition of maximum operation]\label{lem::FuncDecomp}
    It satisfies that for $k\in\{1,\ldots, \alpha\}$, $j\in \{1,\ldots, 2^{\alpha-k}\}$
    \begin{equation*}
        \begin{aligned}
            r(s,\pi^{(k)}_{j}(s))
            &=\max\{r(s,\pi^{(k-1)}_{2j-1}(s)),r(s,\pi^{(k-1)}_{2j}(s))\} \\
            &=\frac{1}{2}\left(r(s,\pi^{(k-1)}_{2j-1}(s))+r(s,\pi^{(k-1)}_{2j}(s))+\Abs{r(s,\pi^{(k-1)}_{2j-1}(s))-r(s,\pi^{(k-1)}_{2j}(s))}\right).
        \end{aligned}
    \end{equation*}
\end{lem}

\begin{proof}
    This comes from the definition and that $\max\{a,b\}=\frac{1}{2}(a+b+\Abs{a-b})$
\end{proof}

Now we start to prove Proposition \ref{prop::ApproxErr4Max} and its Corollary \ref{coro::ApproxErr4Sing}.

\begin{proof}[Proof of Proposition \ref{prop::ApproxErr4Max}]
    This is proved by method of induction. And for simplicity, we only prove the case of $j=1.$
    
    \textbf{Step $0$: } For $k=0$, we can refer to~\cite{ma2019priori} theorem 2.7, which indicates that for each $a\in A$, there is an residual network $f(\cdot, \tilde{\theta}(a))$ of size $(d,m,d+1,1)$ with $\tilde{\theta
        }(a)=\operatorname{vec}\{\ \tilde{u}_a,\tilde{v}_a, \tilde{W}_a,\tilde{U}_a\}$,   $\norm{\tilde{\theta}(a)}_{\fP}\le 12\norm{r(\cdot,a)}_{\fB}$ and 
        \begin{equation}\label{eq::Approx4TwoLay}
            \mathbb{E}_{s \sim  \fD}\left(f(s, \tilde{\theta}(a))-r(s,a)\right)^2\le\frac{16\norm{r(\cdot,a)}_{\fB}^2}{m}.
        \end{equation}
    
    \textbf{Step $1$: }For $k=1$, consider $\pi^{(1)}_{1}(s)$ as an example (the case of $\pi^{(1)}_{j}(s)$ can be proved similarly),  by Proposition~\ref{lem::FuncDecomp}, we see
        \begin{equation*}
            \begin{aligned}
                r(s,\pi^{(1)}_{1}(s))&=\frac{1}{2}\left(r(s,1)+r(s,2)+\Abs{r(s,1)-r(s,2)}\right)\\
                &=\frac{1}{2}\left(r(s,1)+r(s,2)+\sigma(r(s,1)-r(s,2))+\sigma(r(s,2)-r(s,1))\right)
            \end{aligned}
        \end{equation*}
        By the first induction step, there are residual networks $f(s, \tilde{\theta}_1)$, $f(s, \tilde{\theta}_2)$ of size $(d,m,d+1,1)$ such that equation~\eqref{eq::Approx4TwoLay} holds. And by Lemma~\ref{lem::DeepenRN}, there are residual networks $f(s, \breve{\theta}_1)$ and $f(s, \breve{\theta}_2)$ of size $(d,m,d+2,2)$ such that
        \begin{equation}\label{eq::Deepen4TwoLay}
            \begin{aligned}
                f(s, \breve{\theta}_1)=f(s, \tilde{\theta}_1),\ \norm{\breve{\theta}_1}_{\fP}=\norm{\tilde{\theta}_1}_{\fP}\\
                f(s, \breve{\theta}_2)=f(s, \tilde{\theta}_2),\ \norm{\breve{\theta}_2}_{\fP}=\norm{\tilde{\theta}_2}_{\fP}\\
            \end{aligned}
        \end{equation}

        Then consider the residual network representation of $\sigma\left(f(s, \tilde{\theta}_1)-f(s, \tilde{\theta}_2)\right)$.
        By Lem-ma~\ref{lem::CombRN}, there is an residual network $f(s,\bar{\theta}_1)$ of size $(d,2m,2d+2,1)$ such that
        \begin{equation*}
            f(s,\bar{\theta}_1)= f(s, \tilde{\theta}_1)-f(s, \tilde{\theta}_2),\ \norm{\bar{\theta}_1}_{\fP}=\norm{\tilde{\theta}_1}_{\fP}+\norm{\tilde{\theta}_2}_{\fP}.
        \end{equation*}
         And by Lemma~\ref{lem::DeepenRN4Activate}, there is residual network $f(s,\check{\theta}_1)$ of size $(d,2m,2d+3,2)$ such that
        \begin{equation}\label{eq::Deepen4PosiMinus}
            f(s,\check{\theta}_1)= \sigma\left(f(s,\bar{\theta}_1)\right),\ \norm{\check{\theta}_1}_{\fP}=3\norm{\bar{\theta}}_{\fP}=3(\norm{\tilde{\theta}_1}_{\fP}+\norm{\tilde{\theta}_2}_{\fP}).
        \end{equation}
         This argument also holds for $\sigma\left(f(s, \tilde{\theta}_2)-f(s, \tilde{\theta}_1)\right)$ with residual network $f(s,\check{\theta}_2)$ of size $(d, \\ 2m, 2d+3,2)$. Moreover, notice that $\sigma$ is Lipschitz continuous with constant 1, hence
        \begin{equation}\label{eq::Approx4ThreeLayPosiMinus}
            \begin{aligned}
                & \mathbb{E}_{s \sim \fD}\left(f(s,\check{\theta}_1)-\sigma(r(s,1)-r(s,2))\right)^2 \\
                = & \mathbb{E}_{s \sim \fD}\left(\sigma\left(f(s, \tilde{\theta}_1)-f(s, \tilde{\theta}_2)\right)-\sigma \left(r(s,1)-r(s,2) \right) \right)^2\\
                \le & 2\left(\mathbb{E}_{s \sim \fD}\left(f(s, \tilde{\theta}_1)-r(s,1)\right)^2+\mathbb{E}_{s \sim \fD}\left(f(s, \tilde{\theta}_2)-r(s,2)\right)^2\right)\\
                \le &32\frac{\norm{r(\cdot,1)}_{\fB}^2+\norm{r(\cdot,2)}_{\fB}^2}{m}.
            \end{aligned}
        \end{equation}
        By Lemma~\ref{lem::CombRN} there is residual network $f(s,\tilde{\theta}^{(1)})$ of size $(d,6m,6d+10,2)$ such that
        \begin{equation} \label{eq::funcDecomExample}
            f(s,\tilde{\theta}^{(1)})=
            \frac{1}{2}\left(f(s, \breve{\theta}_1)+f(s, \breve{\theta}_2)+f(s,\check{\theta}_1)+f(s,\check{\theta}_2)\right).
        \end{equation}
        Combining equations~\eqref{eq::Deepen4TwoLay}, ~\eqref{eq::Deepen4PosiMinus}, ~\eqref{eq::Approx4ThreeLayPosiMinus} and ~\eqref{eq::funcDecomExample} and using Jensen's inequality
        \begin{equation*}
            \begin{aligned}
                & \mathbb{E}_{s \sim  \fD}\left(f(s,\tilde{\theta}^{(1)})-r(s,\pi^{(1)}_{1}(s))\right)^2 \\
                \leq & 4\mathbb{E}_{s \sim  \fD} \left(\frac{1}{2}(f(s, \breve{\theta}_1) - r(s,1)) \right)^2 + 4\mathbb{E}_{s \sim  \fD} \left(\frac{1}{2}(f(s, \breve{\theta}_2) - r(s,2)) \right)^2 \\
                &~~+ 4\mathbb{E}_{s \sim  \fD} \left(\frac{1}{2}(f(s,\check{\theta}_1) - \sigma(r(s,1)-r(s,2)))\right)^2 \\
                &~~+ 4\mathbb{E}_{s \sim  \fD} \left(\frac{1}{2}(f(s,\check{\theta}_2) - \sigma(r(s,2)-r(s,1)))\right)^2 \\
                \leq &  80\frac{(\norm{r(\cdot,1)}_{\fB}^2+\norm{r(\cdot,2)}_{\fB}^2)}{m},
            \end{aligned}
        \end{equation*}
        and we also have 
        \begin{equation*}
            \begin{aligned}
                \norm{\tilde{\theta}^{(1)}}_{\fP}
                & = \frac{1}{2} \left(\norm{\tilde{\theta}_1}_{\fP}+\norm{\tilde{\theta}_2}_{\fP} + 3(\norm{\tilde{\theta}_1}_{\fP}+\norm{\tilde{\theta}_2}_{\fP}) + 3(\norm{\tilde{\theta}_2}_{\fP}+\norm{\tilde{\theta}_1}_{\fP})\right) \\
                & \leq 42(\norm{r(\cdot,1)}_{\fB}+\norm{r(\cdot,2)}_{\fB}).
            \end{aligned}
        \end{equation*}
        \textbf{Step $k$: } Now, suppose that for $k\ge 2$, there is residual network $f(\cdot, \tilde{\theta}^{(k)})$ of size $(d,p_km,q_kd+s_k,k+1)$ with $p_k=6p_{k-1}$, $q_k=6q_{k-1}$, $s_k=6s_{k-1}+4$,  $p_0=q_0=s_0=1$ such that 
        \begin{equation*}
            \mathbb{E}_{s \sim  \fD}\left(f(s, \tilde{\theta}^{(k)})-r(s,\pi^{(k)}_{j}(s))\right)^2\le 5^{k}\frac{16\sum_{i=1}^{2^k}\norm{r(\cdot,i)}_{\fB}^2}{m}
        \end{equation*}
        with $\norm{\tilde{\theta}^{(k)}}_{\fP}\le (7/2)^{k}\sum_{i=1}^{2^k}12\norm{r(\cdot,i)}_{\fB}.$

        Consider the case for $\pi^{(k+1)}_{1}(s)$: since 
        \begin{equation*}
            \begin{aligned}
                r(s,\pi^{(k+1)}_{1}(s))
            & =\frac{1}{2}\bigl(r(s,\pi^{(k)}_{1}(s))+r(s,\pi^{(k)}_{2}(s)) + \sigma\bigl(r(s,\pi^{(k)}_{1}(s))-r(s,\pi^{(k)}_{2}(s))\bigr) \\
            & ~~~~ +\sigma\bigl(r(s,\pi^{(k)}_{2}(s))-r(s,\pi^{(k)}_{1}(s))\bigr)\bigr).
            \end{aligned}
        \end{equation*}
        By previous induction, there are residual networks $f(s,\tilde{\theta}^{(k)}_1)$ and $f(s,\tilde{\theta}^{(k)}_2)$ of size $(d, \\ p_{k}m, q_{k}d+s_{k},k+1)$ such that 
        \begin{equation*}
            \begin{aligned}
                \mathbb{E}_{s \sim  \fD}\left(f(s,\tilde{\theta}^{(k)}_{1})-r(s,\pi^{(k)}_{1}(s))\right)^2&\le 5^{k}\frac{16\sum_{i=1}^{2^k}\norm{r(\cdot,i)}_{\fB}^2}{m},\\
                \mathbb{E}_{s \sim  \fD}\left(f(s,\tilde{\theta}^{(k)}_{2})-r(s,\pi^{(k)}_{2}(s))\right)^2&\le 5^{k}\frac{16\sum_{i=2^k+1}^{2^{k+1}}\norm{r(\cdot,i)}_{\fB}^2}{m}.
            \end{aligned}
        \end{equation*}
        And by Lemma~\ref{lem::DeepenRN}, there are residual networks $f(s,\breve{\theta}^{(k+1)}_{1})$, $f(s,\breve{\theta}^{(k+1)}_{2})$ of size $(d,p_{k}m,\\ q_{k}d+s_{k}+1,k+2)$ such that
        \begin{equation}\label{eq::Deepen4TwoLayDeep}
            \begin{aligned}
                f(s,\breve{\theta}^{(k+1)}_{1})=f(s,\tilde{\theta}^{(k)}_{1})&,\ \norm{\breve{\theta}^{(k+1)}_{1}}_{\fP}=\norm{\tilde{\theta}^{(k)}_{1}}_{\fP},\\
                f(s,\breve{\theta}^{(k+1)}_{2})=f(s,\tilde{\theta}^{(k)}_{2})&,\ \norm{\breve{\theta}^{(k+1)}_{2}}_{\fP}=\norm{\tilde{\theta}^{(k)}_{2}}_{\fP}.
            \end{aligned}
        \end{equation}
        Still consider the residual network representation of $\sigma\left(f(s,\tilde{\theta}^{(k)}_{1})-f(s,\tilde{\theta}^{(k)}_{2})\right)$. By Lemma~\ref{lem::CombRN} and~\ref{lem::DeepenRN4Activate}, there is residual network $f(s,\bar{\theta}_1)$ of size $(d,2p_{k}m,2q_{k}d+2s_{k},k+1)$ such that
        \begin{equation*}
            f(s,\bar{\theta}^{(k)}_{1})= f(s,\tilde{\theta}^{(k)}_{1})-f(s,\tilde{\theta}^{(k)}_{2}),\ \norm{\bar{\theta}^{(k)}_{1}}_{\fP}=\norm{\tilde{\theta}^{(k)}_{1}}_{\fP}+\norm{\tilde{\theta}^{(k)}_{2}}_{\fP}.
        \end{equation*}
        And by Lemma~\ref{lem::DeepenRN4Activate}, there is residual network $f(s,\check{\theta}^{(k+1)}_{1})$ of size $(d,2p_{k}m,2q_{k}d+2s_{k}+1,k+2)$ such that
        \begin{equation}\label{eq::Deepen4PosiMinusDeep}
            f(s,\check{\theta}^{(k+1)}_{1})= \sigma\left(f(s,\bar{\theta}^{(k)}_{1})\right),\ \norm{\check{\theta}^{(k+1)}_{1}}_{\fP}=3\norm{\bar{\theta}^{(k)}_{1}}_{\fP}=3(\norm{\tilde{\theta}^{(k)}_{1}}_{\fP}+\norm{\tilde{\theta}^{(k)}_{2}}_{\fP}).
        \end{equation}
        This argument holds for $\sigma\left(f(s,\tilde{\theta}^{(k)}_{2})-f(s,\tilde{\theta}^{(k)}_{1})\right)$ with residual network $f(s,\check{\theta}^{(k+1)}_{2})$ of size $(d,2p_{k}m,2q_{k}d+2s_{k}+1,k+2)$ also. Moreover, notice that $\sigma$ is Lipschitz continuous with constant 1, hence
        \begin{equation}\label{eq::Approx4ThreeLayPosiMinusDeep}
            \begin{aligned}
                & \mathbb{E}_{s\sim \fD}\left(f(s,\check{\theta}^{(k+1)}_{1})-\sigma\left(r(s,\pi^{(k)}_{1}(s))-r(s,\pi^{(k-1)}_{2}(s))\right)\right)^2\\
                = & \mathbb{E}_{s \sim \fD}\left(\sigma\left(f(s,\tilde{\theta}^{(k)}_{1})-f(s,\tilde{\theta}^{(k)}_{2})\right)-\sigma\left(r(s,\pi^{(k)}_{1}(s))-r(s,\pi^{(k)}_{2}(s))\right)\right)^2\\
                \le & 2\mathbb{E}_{s \sim \fD}\left(f(s,\tilde{\theta}^{(k)}_{1})-r(s,\pi^{(k)}_{1}(s))\right)^2+2\mathbb{E}_{s \sim S}\left(f(s,\tilde{\theta}^{(k)}_{2})-r(s,\pi^{(k)}_{2}(s))\right)^2\\
                \le & 5^{k}\frac{32\sum_{i=1}^{2^{k+1}}\norm{r(\cdot,i)}_{\fB}^2}{m}.
            \end{aligned}
        \end{equation}
         By Lemma~\ref{lem::CombRN} there is residual network $f(\cdot, \tilde{\theta}^{(k+1)})$ of size $(d,p_{k+1}m,q_{k+1}d+s_{k+1},k+2)$ with
         $p_{k+1}=6p_{k}$, $q_{k+1}=6q_{k}$, $s_{k+1}=6s_{k}+4$ such that
        \begin{equation*}
            f(s, \tilde{\theta}^{(k+1)})=
            \frac{1}{2}\left(f(s,\breve{\theta}^{(k+1)}_1)+f(s,\breve{\theta}^{(k+1)}_2)+f(s,\check{\theta}^{(k+1)}_{1})+f(s,\check{\theta}^{(k+1)}_{2})\right).
        \end{equation*}
        combining equations~\eqref{eq::Deepen4TwoLayDeep}, ~\eqref{eq::Deepen4PosiMinusDeep} and ~\eqref{eq::Approx4ThreeLayPosiMinusDeep}, and using Jensen's inequality
        \begin{equation*}
            \begin{aligned}
                &\mathbb{E}_{s \sim  \fD}\left(f(s, \tilde{\theta}^{(k+1)})-r(s,\pi^{(k+1)}_{1}(s))\right)^2\\
                \le & \mathbb{E}_{s \sim \fD}\left(f(s,\breve{\theta}^{(k+1)}_{1})-r(s,\pi^{(k)}_{1}(s))\right)^2+\mathbb{E}_{s \sim S}\left(f(s,\breve{\theta}^{(k+1)}_{2})-r(s,\pi^{(k)}_{2}(s))\right)^2 \\
                &~~~~~~+\mathbb{E}_{s \sim \fD}\left(f(s,\check{\theta}^{(k+1)}_{1})-\sigma\left(r(s,\pi^{(k)}_{1}(s))-r(s,\pi^{(k)}_{2}(s))\right)\right)^2\\
                &~~~~~~+\mathbb{E}_{s \sim \fD}\left(f(s,\check{\theta}^{(k+1)}_{2})-\sigma\left(r(s,\pi^{(k)}_{2}(s))-r(s,\pi^{(k)}_{1}(s))\right)\right)^2\\
                \le & 5^{k}\frac{16\sum_{i=1}^{2^{k}}\norm{r(\cdot,i)}_{\fB}^2}{m} + 5^{k}\frac{16\sum_{i=2^{k}+1}^{2^{k+1}}\norm{r(\cdot,i)}_{\fB}^2}{m} +4 \times 5^{k}\frac{16\sum_{i=1}^{2^{k+1}}\norm{r(\cdot,i)}_{\fB}^2}{m}\\
                = & 5^{k}\frac{16\sum_{i=1}^{2^{k+1}}\norm{r(\cdot,i)}_{\fB}^2}{m}.
            \end{aligned}
        \end{equation*}
        with 
        \begin{equation*}
            \begin{aligned}
                \norm{\tilde{\theta}^{(k+1)}}_{\fP}
                &= \frac{1}{2}(\norm{\breve{\theta}^{(k+1)}_{1}}_{\fP}+\norm{\breve{\theta}^{(k+1)}_{2}}_{\fP}+\norm{\check{\theta}^{(k+1)}_1}_{\fP}+\norm{\check{\theta}^{(k+1)}_{2}}_{\fP})\\
                &=\frac{7}{2}\norm{\tilde{\theta}^{(k)}_{1}}_{\fP}+\frac{7}{2}\norm{\tilde{\theta}^{(k)}_{2}}_{\fP}\le (7/2)^{k+1}\sum_{i=1}^{2^{k+1}}12\norm{r(\cdot,i)}_{\fB}.
            \end{aligned}
        \end{equation*}
        Hence completes the induction.
\end{proof}

\begin{proof}[Proof of Corollary \ref{coro::ApproxErr4Sing}]
    First of all, the concrete form of the approximation error bound we are trying to prove is
    \begin{equation*}
        \mathbb{E}_{s \sim  \fD}\left( f(s, \tilde{\theta}(a))-f^*(s,a)\right)^2 \le \frac{32}{(1-\gamma)^2}\Abs{A}^3\frac{\norm{r(\cdot,\cdot)}_{\fB}^2}{m}+\frac{32\norm{r(\cdot,a)}_{\fB}^2}{m},
    \end{equation*}
    with $\norm{\tilde{\theta}(a)}_{\fP}\le 12\norm{r(\cdot,a)}_{\fB} + \frac{12}{1-\gamma}\Abs{A}^{5/2}\norm{r(\cdot,\cdot)}_{\fB}.$ 

   For each $a\in A$, by~\cite{ma2019priori}, Theorem 2.7, there is an residual network $g(s,\theta(a))$ of size $(d,m,d+1,1)$ such that  
    \begin{equation}\label{eq::Approx4PartI}
        \mathbb{E}_{s \sim \fD}\left(g(s,\theta(a))-r(s,a)\right)^2\le\frac{16\norm{r(\cdot,a)}_{\fB}^2}{m},
    \end{equation}
    with $\norm{\theta(a)}_{\fP}\le 12\norm{r(\cdot,a)}_{\fB}$. By Lemma~\ref{lem::DeepenRN}, there is residual network $g_0(s,\breve{\theta}(a))$ of size $(d,m,d+\alpha+1,\alpha+1)$ such that
    \begin{equation*}
        g_0(s,\breve{\theta}(a))=g(s,\theta(a)),
    \end{equation*} 
    with $\norm{\breve{\theta}(a)}_{\fP}=\norm{\theta(a)}_{\fP}$. By Theorem~\ref{prop::ApproxErr4Max} and that $r(s,\pi(s))=r(s,\pi^{(\alpha)}_{1}(s))$, there are $D\in \sN^+$ and residual network $h(s,\tilde{\theta}^{(\alpha)})$ of size $(d,6^{\alpha}m,D,\alpha+1)$ such that
    \begin{equation}\label{eq::Approx4PartII}
        \mathbb{E}_{s \sim  \fD}\left(h(s,\tilde{\theta}^{(\alpha)})-r(s,\pi(s))\right)^2 \le 5^{\alpha}\frac{16\sum_{i=1}^{\Abs{A}}\norm{r(\cdot,i)}_{\fB}^2}{m},
    \end{equation}
    with $\norm{\tilde{\theta}^{(\alpha)}}_{\fP}\le (7/2)^{\alpha}\sum_{i=1}^{\Abs{A}}12\norm{r(\cdot,i)}_{\fB}$. Thus by Lemma~\ref{lem::CombRN}, there is residual network $f(s, \tilde{\theta}(a))$ of size $(d,(6^{\alpha}+1)m,D+d+\alpha+1,\alpha+1)$ such that
    \begin{equation}\label{eq::Decomp4ApproxRN}
        f(s, \tilde{\theta}(a))=g_0(s,\breve{\theta}(a))+\frac{\gamma}{1-\gamma}h(s,\tilde{\theta}^{(\alpha)}),
    \end{equation}
    with 
    \begin{equation*}
        \begin{aligned}
            \norm{\tilde{\theta}(a)}_{\fP}
            &=\norm{\breve{\theta}(a)}_{\fP}+\frac{\gamma}{1-\gamma}\norm{\tilde{\theta}^{(\alpha)}}_{\fP}\le 12\norm{r(\cdot,a)}_{\fB}+\frac{1}{1-\gamma}(7/2)^{\alpha}\sum\nolimits_{i=1}^{\Abs{A}}12\norm{r(\cdot,i)}_{\fB}\\
            & \leq 12\norm{r(\cdot,a)}_{\fB} + \frac{12}{1-\gamma}\Abs{A}^{5/2}\norm{r(\cdot,\cdot)}_{\fB}.
        \end{aligned}
    \end{equation*}
    Moreover, \eqref{eq::Solf*},  \eqref{eq::Approx4PartI} and~\eqref{eq::Approx4PartII} together lead to
    \begin{equation*}
        \begin{aligned}
            & \mathbb{E}_{s \sim  \fD}\left(f(s, \tilde{\theta}(a))-f^*(s,a)\right)^2 \\
            \leq & 2\mathbb{E}_{s \sim  \fD}\left(g_0(s,\breve{\theta}(a))-r(s,a)\right)^2 + 2\mathbb{E}_{s \sim  \fD}\left(\frac{\gamma}{1-\gamma} \left( h(s,\tilde{\theta}^{(\alpha)}) - r(s,\pi(s)) \right)\right)^2 \\
            \leq & \frac{32}{(1-\gamma)^2}\Abs{A}^3\frac{\norm{r(\cdot,\cdot)}_{\fB}^2}{m}+\frac{32\norm{r(\cdot,a)}_{\fB}^2}{m},
        \end{aligned}
    \end{equation*}
    where $5^{\alpha} \leq |A|^3$. This completes the proof.
\end{proof}

From Corollary \ref{coro::ApproxErr4Sing}, we can estimate the approximation error of Bellman effective loss.

\begin{proof}[Proof of Proposition \ref{prop::ApproxErr4Sum}]
    First of all, the concrete form of the approximation error bound we are trying to prove is
    \begin{equation*}
        \begin{aligned}
            \tilde{\fR}_{\fD}(\Tilde{\Theta}) \le \frac{1}{(1-\gamma)^2}(\Abs{A}+1)(\Abs{A}^4+1)\frac{32\norm{r(\cdot,\cdot)}_{\fB}^2}{m}.
        \end{aligned}
    \end{equation*}
    with $\norm{\tilde{\Theta}}_{\fP} \le\frac{12}{1-\gamma}(\Abs{A}^{7/2} + \Abs{A}^{1/2})\norm{r(\cdot,\cdot)}_{\fB}.$ 

    Insert effective solution \eqref{eq::Solf*} into Bellman effective loss we have
    \begin{equation}\label{eq::Deform}
        \begin{aligned}
            & f(s, \theta(a))-r(s,a)-\gamma \max_{a'\in A}f(s, \theta(a')) \\
            =& \left(f(s, \theta(a))-f^*(s,a)\right)-\gamma \left(\max_{a'\in A}f(s,\theta(a'))-\max_{a''\in A}f^*(s,a'')\right).
        \end{aligned}
    \end{equation}
    Note that for each $a\in A$ we have
    \begin{equation*}
        f(s, \theta(a))-\max_{a''\in A}f^*(s,a'')\le f(s, \theta(a))-f^*(s,a)\le \max_{a'\in A}\Abs{f(s,\theta(a'))-f^*(s,a')}.
    \end{equation*}
    So we have the following relationship 
    \begin{equation} \label{eq::MaxContro}
        \Abs{\max_{a'\in A}f(s, \theta(a'))-\max_{a''\in A}f^*(s,a'')}\le \max_{a'\in A}\Abs{f(s,\theta(a'))-f^*(s,a')}.
    \end{equation}

    From Corollary~\ref{coro::ApproxErr4Sing}, we have for each $a\in A$
    \begin{equation*}
         \mathbb{E}_{s\sim \fD}\left(f(s, \tilde{\theta}(a))-f^*(s,a)\right)^2\le \frac{1}{(1-\gamma)^2}\Abs{A}^3\frac{32\norm{r(\cdot,\cdot)}_{\fB}^2}{m}+\frac{32\norm{r(\cdot,a)}_{\fB}^2}{m},
    \end{equation*}
    this combing equations~\eqref{eq::Deform} and~\eqref{eq::MaxContro} leads to 
    \begin{equation*}
        \begin{aligned}
            \tilde{\fR}(\tilde{\Theta})&=\frac{1}{2}\mathbb{E}_{s\sim \fD, a \sim \fU}\left(\left(f(s, \theta(a))-f^*(s,a)\right)-\gamma \left(\max_{a'\in A}f(s, \theta(a'))-\max_{a''\in A}f^*(s,a'')\right)\right)^2\\
            &\le \frac{1}{|A|} \sum_{a\in A}\left(\mathbb{E}_{s\sim \fD}\left(f(s, \theta(a))-f^*(s,a)\right)^2+\Abs{A}\sum_{a\in A}\left(\mathbb{E}_{s\sim \fD}\left(f(s, \theta(a))-f^*(s,a)\right)^2\right)\right)\\
            & \leq \frac{1}{\Abs{A}}\left(\frac{1}{(1-\gamma)^2}\Abs{A}^4+1\right)\frac{32\norm{r(\cdot,\cdot)}_{\fB}^2}{m} + \Abs{A}\left(\frac{1}{(1-\gamma)^2}\Abs{A}^4+1\right)\frac{32\norm{r(\cdot,\cdot)}_{\fB}^2}{m} \\
            & \leq \frac{1}{(1-\gamma)^2}(\Abs{A}+1)(\Abs{A}^4+1)\frac{32\norm{r(\cdot,\cdot)}_{\fB}^2}{m}.
        \end{aligned}
    \end{equation*}
    Also, we have
    \begin{equation*}
        \begin{aligned}       
        \norm{\tilde{\Theta}}_{\fP} & = \sum_{a \in A}\norm{\tilde{\theta}(a)}_{\fP}\le \sum\nolimits_{i=1}^{\Abs{A}}12\norm{r(\cdot,i)}_{\fB}+\frac{12}{1-\gamma}\Abs{A}^{7/2}\norm{r(\cdot,\cdot)}_{\fB} \\
        & \leq \frac{12}{1-\gamma}(\Abs{A}^{7/2} + \Abs{A}^{1/2})\norm{r(\cdot,\cdot)}_{\fB}.
        \end{aligned}
    \end{equation*}
\end{proof}

\subsection{\emph{a priori} generalization error for Bellman optimal loss minimization problem}

Than-ks to the Lipschitz property of the transition function, as stated in Assumption \ref{assum::SemiGroup}, we have the following theorem.

\begin{proof}[Proof of Theorem \ref{thm::apprErrorOrigin}]
    First of all, the concrete form of the approximation error bound we are trying to prove is
    \begin{equation*}
        \begin{aligned}
            \fR_{\fD}(\tilde{\Theta}) & \leq \frac{1}{(1-\gamma)^2} \left( 256 \Abs{A}^{5} \frac{\norm{r(\cdot,\cdot)}_{\fB}^2}{m}  +   576 \gamma^2 C_T^2 (\Delta t)^2  \Abs{A}^{7} \norm{r(\cdot,\cdot)}^2_{\fB} \right).
        \end{aligned}
    \end{equation*}
    with $\norm{\tilde{\Theta}}_{\fP} \leq \frac{24}{1-\gamma}\Abs{A}^{7/2}\norm{r(\cdot,\cdot)}_{\fB}.$

    To get the relation between $\tilde{\fR}_{\fD}(\tilde{\Theta})$ and $\fR_{\fD}(\tilde{\Theta})$, we define 
    \begin{equation*}
        \Delta(s,a) = \gamma(\max_{a' \in A} f(s,\tilde{\theta}(a')) - \max_{a'' \in A} f(\tilde{g}(s,a, \Delta t),\tilde{\theta}(a''))).
    \end{equation*}
    Then we have the following relationship
    \begin{align*}
        \fR_{\fD}(\tilde{\Theta}) & = \frac{1}{2}\mathbb{E}_{s \sim \fD, a \sim \fU} \left(f(s, \tilde{\theta}(a))-r(s,a)-\gamma\max_{a'\in A}f(s,\tilde{\theta}(a')) + \Delta(s,a) \right)^2 \\
        & \leq \frac{1}{2}\mathbb{E}_{s \sim \fD, a \sim \fU} \left( 2\left(f(s, \tilde{\theta}(a))-r(s,a)-\gamma\max_{a'\in A}f(s,\tilde{\theta}(a')) \right)^2 + 2\Abs{\Delta(s,a)}^2 \right) \\
        & \leq  2\tilde{\fR}_{\fD}(\tilde{\Theta}) + \left(\sup_{s' \in S, a''\in A}\Abs{\Delta(s,a)}\right)^2.
    \end{align*}
    Now we start to estimate $\sup_{s \in S, a\in A}\Abs{\Delta(s,a)}$,
    \begin{equation*}
        \begin{aligned}
            \sup_{s \in S, a\in A}\Abs{\Delta(s,a)} 
            & = \sup_{s \in S, a\in A} \left| \gamma \left(\max_{a' \in A} f(s,\tilde{\theta}(a')) - \max_{a'' \in A} f(\tilde{g}(s,a,\Delta t),\tilde{\theta}(a''))\right) \right| \\
            & \leq \sup_{s \in S, a\in A} \left| \gamma \max_{a' \in A}\Abs{ f(s,\tilde{\theta}(a')) - f(\tilde{g}(s,a,\Delta t),\tilde{\theta}(a'))} \right| \\
            & \leq \gamma \sum_{a' \in A} \sup_{s \in S,a\in A} \Abs{ f(s,\tilde{\theta}(a')) - f(\tilde{g}(s,a,\Delta t),\tilde{\theta}(a'))}.
        \end{aligned}
    \end{equation*}
    We define $\Gamma(a) = \Abs{u^{\T}_a}(I+3\Abs{U^{[L]}_a}\Abs{W^{[L]}_a})\ldots(I+3\Abs{U^{[1]}_a}\Abs{W^{[1]}_a})\Abs{V_a} \in \mathbb{R}^d$. For all $s \in S, a\in A$, from Assumption \ref{assum::SemiGroup} we have
    \begin{equation*}
        \begin{aligned}
            \sup_{s \in S,a\in A} \Abs{ f(s,\tilde{\theta}(a')) - f(\tilde{g}(s,a,\Delta t),\tilde{\theta}(a'))} 
            & \leq \sup_{s \in S,a\in A} \left(\Gamma(a') |s| - \Gamma(a') |\tilde{g}(s,a,\Delta t)| \right)\\
            & \leq \sup_{s \in S,a\in A} \left(\Gamma(a') |s - \tilde{g}(s,a,\Delta t)| \right)\\
            &\leq C_T \Delta t \norm{\tilde{\theta}(a')}_{\fP}.
        \end{aligned}
    \end{equation*}
    We have $\sup_{s \in S, a\in A}\Abs{\Delta(s,a)}  \leq \gamma C_T \Delta t \norm{\tilde{\Theta}}_{\fP}$. So the relation between $\tilde{\fR}_{\fD}(\tilde{\Theta})$ and $\fR_{\fD}(\tilde{\Theta})$ is
    \begin{equation}\label{eqn::original_loss_relation}
        \fR_{\fD}(\tilde{\Theta}) \leq 2\tilde{\fR}_{\fD}(\tilde{\Theta}) + \gamma^2 C_T^2 (\Delta t)^2 \norm{\tilde{\Theta}}_{\fP}^2.
    \end{equation}
    From Theorem \ref{prop::ApproxErr4Sum}, we have
    \begin{equation}\label{eqn::theta_bound}
        \norm{\tilde{\Theta}}_{\fP} \le  \frac{12}{1-\gamma}(\Abs{A}^{7/2} + \Abs{A}^{1/2})\norm{r(\cdot,\cdot)}_{\fB} \leq \frac{24}{1-\gamma}\Abs{A}^{7/2}\norm{r(\cdot,\cdot)}_{\fB}.
    \end{equation}
    We have simplified the expression $\tilde{\fR}_{\fD}(\tilde{\Theta})$ as follows
    \begin{equation}\label{eqn::effective_loss_bound}
        \begin{aligned}
            \tilde{\fR}_{\fD}(\tilde{\Theta}) & \leq \frac{1}{(1-\gamma)^2}(\Abs{A}+1)(\Abs{A}^4+1)\frac{32\norm{r(\cdot,\cdot)}_{\fB}^2}{m} \leq \frac{128}{(1-\gamma)^2} \Abs{A}^{5}\frac{\norm{r(\cdot,\cdot)}_{\fB}^2}{m}.
        \end{aligned}    
    \end{equation}
    So by combining  \eqref{eqn::original_loss_relation},  \eqref{eqn::theta_bound} and \eqref{eqn::effective_loss_bound}, we have the approximation error of $\fR_{\fD}(\tilde{\Theta})$.
\end{proof}

Now we start to give the \emph{a posteriori} estimates of the generalization error. In order to estimate the Rademacher complexity of $\fG_{M}$, we need to estimate the complexity of $\fZ^{\max}_M$. 

\begin{lem}[function space with maximum operation]\label{lem::RNSpace4Max}
    Let $\fF$ defined in \eqref{eq::fF} consist of $f(\cdot,\theta(\cdot))$ with size $(d,m,D,L)$, then there exists $\tilde{m}, \tilde{D}, \tilde{L}$ and $\tilde{\fF}$ which consists of all real-valued residual network of size $(d,\tilde{m}, \tilde{D}, \tilde{L})$ such that $\fZ^{\max}\subseteq \tilde{\fF}$ and that $\fZ^{\max}_M\subseteq \tilde{\fF}_{R}$, where $\tilde{\fF}_{R}=\left\{\tilde{f}(\cdot, \theta)\in \tilde{\fF}\colon\, \norm{\theta}_{\fP}\le R \right\}$ and $R=\Abs{A}^3 M$.
\end{lem}
\begin{proof}
    For any $s \in S$ and given $f \in \fF$, define $\Pi^{(0)}=\{\Pi^{(0)}_{1}(s),\ldots, \Pi^{(0)}_{2^{\alpha}}(s)\}$ such that
    \begin{equation*}
        \Pi^{(0)}_{j}(s)=f(s,\theta(j)), \text{ for all } j\in \{1,\ldots,2^{\alpha}\}.
    \end{equation*} 
    Also, for $k\in \{1,\ldots, \alpha\}$, define $\Pi^{(k)}=\{\Pi^{(k)}_{1}(s),\ldots, \Pi^{(k)}_{2^{\alpha-k}}(s)\}$ such that 
    \begin{equation*}
        \Pi^{(k)}_{j}(s)=\max\{ \Pi^{(k-1)}_{2j-1}(s),\Pi^{(k-1)}_{2j}(s)\}, \text{ for all } j\in \{1,\ldots,2^{\alpha-k}\}.
    \end{equation*}
    And it is easy to see that $\max_{a'\in A}f(s, \theta(a'))=\Pi^{(\alpha)}_{1}(s).$

    For every $k\in \{0,1,\ldots, \alpha\}$, we donate $\tilde{\fF}^k$ as the function space consists of all real-valued residual network of size $(d,m_k,D_k,L_k)$, where $\tilde{\fF}^0$ contains all the residual network with size $(d, m, D, L)$ and $m_k = 6 m_{k-1}, D_k = 6D_{k-1} + 4, L_k = L_{k-1} + 1, m_0=m, D_0=D, L_0 = L$.
    
    We prove by induction to show that $\Pi^{(k)}_{j}(s)\in \tilde{\fF}^k$ for $j\in \{1,\ldots, 2^{\alpha-k}\}$, and if $f \in \fF_{M}$, which indicates $f(\cdot, \theta(a)) \in \hat{\fF}_M$ for all $a \in A$, then $\Pi^{(k)}_{j}(s)\in \tilde{\fF}^k_{7^kQ}=\{\tilde{f}(\cdot, \theta)\in \tilde{\fF}^k: \norm{\theta}_{\fP}\le 7^kM\}$.
    
    \textbf{Step $0$: } If $k=0$, we can choose $\tilde{\fF}^0=\hat{\fF}$ and $R=M$.  
    
    \textbf{Step $k$: } Suppose that $k\ge 1$, $\Pi^{(k)}_{j}(s) \in \tilde{\fF}^k$ and $\Pi^{(k)}_{j}(s)\in \tilde{\fF}^k_{7^k M}=\{\tilde{f}(\cdot, \theta)\in \tilde{\fF}^k: \norm{\theta}_{\fP} \le 7^k M\}$. Consider the case of $k+1$. We use the same argument in Proposition \ref{prop::ApproxErr4Max} here. Since 
    \begin{equation*}
        \begin{aligned}
            \Pi^{(k+1)}_{j}(s)
            &=\max\{ \Pi^{(k)}_{2j-1}(s),\Pi^{(k)}_{2j}(s)\}\\
            &=\frac{1}{2}\left(\Pi^{(k)}_{2j-1}(s)+\Pi^{(k)}_{2j}(s)+\sigma\bigl(\Pi^{(k)}_{2j-1}(s)-\Pi^{(k)}_{2j}(s)\bigr)  +\sigma\bigl(\Pi^{(k)}_{2j}(s)-\Pi^{(k)}_{2j-1}(s)\bigr)\right),
        \end{aligned}
    \end{equation*}
    and that $\Pi^{(k)}_{2j-1}(s),\Pi^{(k)}_{2j}(s)\in \tilde{\fF}^k$. By Lemmas~\ref{lem::CombRN},~\ref{lem::DeepenRN} and~\ref{lem::DeepenRN4Activate}, there is $m_{k+1},D_{k+1},L_{k+1}$ and an residual network $g(s, \breve{\theta}^{(k+1)}_{j})$ of size $(d,m_{k+1},D_{k+1},L_{k+1})$, which indicates $\Pi^{(k+1)}_{j}(s) \in \tilde{\fF}^{k+1}$, such that
    \begin{equation*}
        g(s, \breve{\theta}^{(k+1)}_{j})=\max\{ \Pi^{(k)}_{2j-1}(s),\Pi^{(k)}_{2j}(s)\}=\Pi^{(k+1)}_{j}(s).
    \end{equation*}
    Moreover, if $\Pi^{(k)}_{2j-1}(s),\Pi^{(k)}_{2j}(s)\in \tilde{\fF}^k_{7^k M}$, by Lemmas~\ref{lem::CombRN},~\ref{lem::DeepenRN} and~\ref{lem::DeepenRN4Activate} again, $\norm{\breve{\theta}^{(k+1)}_j}_{\fP}\le 7^{k+1} M$.
    Thus, by induction, there is $\tilde{\fF}$ which consists of all real-valued residual network of size $(d,\tilde{m}, \tilde{D}, \tilde{L})$ such that $\fZ^{\max}\subseteq\tilde{\fF}$ and that $\fZ^{\max}_M\subseteq\tilde{\fF}_{7^{\alpha}M}\subseteq\tilde{\fF}_{\Abs{A}^3 M}$
\end{proof}

The Rademacher complexity of $\fG_{M}$ can be bounded by the addition of $\fZ^{\max}_M$ and $\fF_M$.
\begin{lem}[Rademacher complexity of function space]\label{lem::Rad}
    Given $M>0$, the Rademacher complexity of $\fG_{M}$ over a set of $n$ i.i.d. samples from $\fD \times \fU$, donated as $\fS = \{s_i, a_i, s_i', r\}_{i=1}^{n}$, has an upper bound
    \begin{equation*}
        \operatorname{Rad}(\fG_{M})\le 6\Abs{A}^3 M\sqrt{\frac{2\ln (2d)}{n}}.
    \end{equation*}
\end{lem}

\begin{proof}
    From~\cite{ma2019priori} Theorem 2.7, given any $n' \in N^+$, $a \in A$ and a sample $\{s_i\}_{i=1}^{n'}$, we have
    \begin{equation*}
        \mathbb{E}_{\tau} \sup_{f \in \fF_M} \sum_{i=1}^{n'} \tau_i f(s_i,a) \leq 3M\sqrt{2\ln(2d)} \times \sqrt{n'}.
    \end{equation*}
    From Lemma \ref{lem::RNSpace4Max}, for each $f \in \fF_{M}$, there exist an residual network $g(\cdot, \theta) \in \tilde{\fF}_{|A|^3 M}$ that has $g(s_i', \theta)=\max_{a' \in A} f(s_i',a')$, then given any $n'' \in N^+$ and a sample $\{s_i'\}_{i=1}^{n''}$, we have
    \begin{equation*}
        \mathbb{E}_{\tau} \sup_{f \in \fF_M} \sum_{i=1}^{n''} \tau_i \max_{a' \in A} f(s_i',a') = \mathbb{E}_{\tau} \sup_{g \in \tilde{\fF}_{|A|^3 M}} \sum_{i=1}^{n''} \tau_i g(s_i',\theta) \leq 3\Abs{A}^3 M\sqrt{2\ln(2d)} \times \sqrt{n''}.
    \end{equation*}
    then we can obtain
    \begin{equation*}
        \begin{aligned}
            \operatorname{Rad}(\fG_{M})
            \leq & \frac{1}{n}\mathbb{E}_{\tau} \left[ \mathbb{E}_{\fS} \left[ \sup_{f \in \fF_M} \sum_{i=1}^{n} \tau_i f(s_i,a_i)\right]\right] + \gamma\mathbb{E}_{\fS} \left[ \frac{1}{n} \mathbb{E}_{\tau} \left[ \sup_{f \in \fF_M} \sum_{i=1}^{n} \tau_i \max_{a' \in A} f(s_i', a') \right]\right] \\
            \leq & \frac{1}{n}\sum_{a \in A}\mathbb{E}_{\tau} \left[ \mathbb{E}_{\fS} \left[ \sup_{f \in \fF_M} \sum_{1 \leq i \leq n, a_i=a}  \tau_i f(s_i,a)\right]\right] + 3\gamma\Abs{A}^3 M \sqrt{\frac{2\ln(2d)}{n}} \\
            \leq & \frac{3 M \sqrt{2\ln(2d)}}{n} \sum_{a \in A} \mathbb{E}_{\tau} \left[ \mathbb{E}_{\fS}\left[ \sqrt{ \left\{ i, a_i=a \right\}}\right] \right] + 3\Abs{A}^3 M\sqrt{\frac{2\ln(2d)}{n}} \\
            \leq & 3 (\Abs{A}^3 + \sqrt{|A|})M\sqrt{\frac{2\ln (2d)}{n}} \leq 6\Abs{A}^3 M\sqrt{\frac{2\ln(2d)}{n}}.
        \end{aligned}
    \end{equation*}
\end{proof}

Combining Lemma \ref{lem::Rad} and Theorem \ref{thm::RadeandGenErr}, we have the following result.
\begin{proof} [Proof of Theorem \ref{thm::PosterGenBoun}]
    First of all, the concrete form of the \emph{a posteriori} generalization error bound we are trying to prove is
    \begin{equation*}
        \Abs{\fR_{\fD}(\Theta)-\fR_{\fS}(\Theta)}\le \frac{(\norm{\Theta}_{\fP}+1)^2}{\sqrt{n}}\left(36 |A|^3\sqrt{2\ln (2d)}+\frac{9}{2}\left( \ln (4(\norm{\Theta}_{\fP}+1)) + \ln(\Abs{A}/\delta) \right)\right).
    \end{equation*}

    Define 
    \begin{equation*}
        \begin{aligned}
            \fH&:=\left\{ \frac{1}{2}\left(f(s, \theta(a))-r(s,a)-\gamma\max_{a'\in A}f(s', \theta(a'))\right)^2\colon\ f(\cdot,\theta(\cdot))\in \fF, a \in A\right\}, \\
            \fH_{M}&:=\left\{ \frac{1}{2}\left(f(s, \theta(a))-r(s,a)-\gamma\max_{a'\in A}f(s', \theta(a'))\right)^2\colon\ f(\cdot,\theta(\cdot))\in \fF_M, a \in A \right\},
        \end{aligned}
    \end{equation*} 
    then $\fH=\bigcup_{M=1}^{\infty}\fH_{M}$. Note that for all $s \in S \subseteq [0,1]^d$
    \begin{equation*}
        \sup_{s\in S, a \in A}\Abs{f(s, \theta(a))}\le \sup_{s\in S, a \in A}\Abs{u^{\T}_a}(I+3\Abs{U^{[L]}_a}\Abs{W^{[L]}_a})\ldots(I+3\Abs{U^{[1]}_a}\Abs{W^{[1]}_a})\Abs{V_a}\Abs{s}\le \norm{\Theta}_{\fP},
    \end{equation*}
    thus for functions in $\fH_M$, since from Assumption \ref{assump::RewardFunctionBarron} we know $\Abs{r(s,a)}\le 1$, we have
    \begin{equation*}
        \begin{aligned}
            \frac{1}{2}\left(f(s, \theta(a))-r(s,a)-\gamma\max_{a'\in A}f(s', \theta(a'))\right)^2 \le \frac{1}{2}(1+2\sup_{s\in  S,a\in A}\Abs{f(s, \theta(a))}) \le \frac{9}{2}M^2,
        \end{aligned}
    \end{equation*}
    for all $s \in  S$ and all $M \ge 1$. Moreover, since $l(\cdot,r(s,a)):=(\cdot-r(s,a))^2$ is a Lipschitz function with Lipschitz constant no more that $2M+1$, combining Lemma \ref{lem::Rad} and Lemma \ref{lem::RadContrac} we have
    \begin{equation*}
        \begin{aligned}
            \operatorname{Rad}(\fH_{M})
            &\le \left( 2M + 1 \right)\operatorname{Rad}(\fF_{M})\\
            & \leq \left( 2M + 1 \right) 6 |A|^3 M\sqrt{\frac{2\ln (2d)}{n}} \\
            &\le 18 |A|^3 M^2\sqrt{\frac{2\ln (2d)}{n}} ,\ M\ge 1.
        \end{aligned}   
    \end{equation*}
    By Theorem~\ref{thm::RadeandGenErr}, this directly leads to that, for any $\delta\in (0,1)$, there is $1-\delta_M$ over the choice of $S$ with $\delta_M=\frac{6\delta}{\Abs{A}\pi^2M^2}$, where $\pi$ represents circumference ratio here. Thus we have
    \begin{equation*}
        \sup_{\norm{\Theta}_{\fP} \leq M} \Abs{\fR_{\fS}(\Theta)-\fR_{\fD}(\Theta)}\le \frac{36 |A|^3 M^2\sqrt{2\ln (2d)}}{\sqrt{n}}+\frac{9 M^2}{2}\sqrt{\frac{\ln (\Abs{A}\pi^2M^2/3\delta)}{2n}},
    \end{equation*}
    and we can choose $M$ such that  $\norm{\Theta}_{\fP} \le M \le \norm{\Theta}_{\fP}+1$, where $\norm{\Theta}_{\fP}$ is the path norm of the parameter. Hence we have with probability $1-\frac{\delta}{\Abs{A}}$ over the choice of $\fS$
    \begin{equation*}
        \begin{aligned}
             \Abs{\fR_{\fD}(\Theta)-\fR_{\fS}(\Theta)}
             &\le \frac{36 |A|^3 M^2\sqrt{2\ln (2d)}}{\sqrt{n}}+\frac{9M^2}{2}\sqrt{\frac{\ln (\Abs{A}\pi^2M^2/3\delta)}{2n}}\\
             &\le \frac{(\norm{\Theta}_{\fP}+1)^2}{\sqrt{n}}\left(36 |A|^3\sqrt{2\ln (2d)}+\frac{9}{2}\left( \ln (4(\norm{\Theta}_{\fP}+1)) + \ln(\Abs{A}/\delta) \right)\right).
        \end{aligned}
    \end{equation*}
    The last inequality holds due to
    \begin{equation*}
        \begin{aligned}
            \sqrt{\frac{\ln (\Abs{A}\pi^2(\norm{\Theta}_{\fP}+1)^2)/3\delta)}{2}}
            & \leq \sqrt{\ln (\pi(\norm{\Theta}_{\fP}+1)) + \frac{\ln(\Abs{A}/\delta)}{2}} \\
            & \leq \ln (\pi(\norm{\Theta}_{\fP}+1)) + \frac{\ln(\Abs{A}/\delta)}{2} \\
            & \leq \ln (4(\norm{\Theta}_{\fP}+1)) + \ln(\Abs{A}/\delta)\\
        \end{aligned}
    \end{equation*}
    The second inequality hold due to $\ln (\pi(\norm{\Theta}_{\fP}+1)) \geq 1$ and $\ln(\Abs{A}/\delta) \geq 0$.
\end{proof}

Combining Theorem \ref{thm::apprErrorOrigin} and Theorem \ref{thm::PosterGenBoun}, we can have the final result.

\begin{proof}[Proof of Theorem \ref{thm::AprioriGenErrBoun}]
    First of all, the concrete form of the \emph{a priori} generalization error bound we are trying to prove is
    \begin{equation*}
        \begin{aligned}
            \fR_{\fD}(\hat{\Theta}) 
            & \leq \frac{1}{(1-\gamma)^2} \left( 256 \Abs{A}^{5} \frac{\norm{r(\cdot,\cdot)}_{\fB}^2}{m}  +   576 \gamma^2 C_T^2 (\Delta t)^2  \Abs{A}^{7} \norm{r(\cdot,\cdot)}^2_{\fB} \right) \\
            &~~~+\frac{576\norm{r(\cdot,\cdot)}_{\fB}^2+2}{(1-\gamma)^2\sqrt{n}}\Abs{A}^7(\lambda+1)\left(72 |A|^3 \sqrt{2\ln (2d)} \right.\\
            &~~~+\left. 9 \ln(4(\frac{24}{1-\gamma}\Abs{A}^{7/2}\norm{r(\cdot,\cdot)}_{\fB}+1)) +  9\ln(\Abs{A}/\delta) \right).
        \end{aligned}
    \end{equation*}

    Note that
    \begin{equation*}
        \fR_{\fD}(\hat{\Theta})=\fR_{\fD}(\tilde{\Theta})+[\fR_{\fD}(\hat{\Theta})-\fJ_{S,\lambda}(\hat{\Theta})]+[\fJ_{S,\lambda}(\hat{\Theta})-\fJ_{S,\lambda}(\tilde{\Theta})]+[\fJ_{S,\lambda}(\tilde{\Theta})-\fR_{\fD}(\tilde{\Theta})].
    \end{equation*}
    By definition, $\fJ_{S,\lambda}(\hat{\Theta})-\fJ_{S,\lambda}(\tilde{\Theta})\le 0$ and Theorem~\ref{thm::apprErrorOrigin}, we have
    \begin{equation}\label{eq::GenErr}
        \begin{aligned}
            \fR_{\fD}(\hat{\Theta}) & \le \frac{1}{(1-\gamma)^2} \left( 256 \Abs{A}^{5} \frac{\norm{r(\cdot,\cdot)}_{\fB}^2}{m}  +   576 \gamma^2 C_T^2 (\Delta t)^2  \Abs{A}^{7} \norm{r(\cdot,\cdot)}^2_{\fB} \right)\\
            &~~~~+ [\fR_{\fD}(\hat{\Theta})-\fJ_{S,\lambda}(\hat{\Theta})]+[\fJ_{S,\lambda}(\tilde{\Theta})-\fR_{\fD}(\tilde{\Theta})].
        \end{aligned}
    \end{equation}
    By Theorem~\ref{thm::PosterGenBoun}, we have with probability at least $1-\delta/2$,
    \begin{equation}\label{eq::Risk&EmpPena}
        \begin{aligned}
             & \fR_{\fD}(\hat{\Theta})-\fJ_{S,\lambda}(\hat{\Theta}) \\
            = & \fR_{\fD}(\hat{\Theta})-\fR_{\fS}(\hat{\Theta})-\frac{\lambda}{\sqrt{n}}\norm{\hat{\Theta}}_{\fP}^2\ln (4(\norm{\hat{\Theta}}_{\fP}+1))\\
            \le & \frac{2(\norm{\hat{\Theta}}_{\fP}^2+1)}{\sqrt{n}}\left(36|A|^3\sqrt{2\ln (2d)}+\frac{9}{2}\left( \ln (4(\norm{\hat{\Theta}}_{\fP}+1)) + \ln(\Abs{A}/\delta) \right)\right)\\
            &~~~-\frac{\lambda}{\sqrt{n}}\norm{\hat{\Theta}}_{\fP}^2\ln (4(\norm{\hat{\Theta}}_{\fP}+1))\\
            \leq & \frac{\norm{\hat{\Theta}}_{\fP}^2}{\sqrt{n}}\ln(4(\norm{\hat{\Theta}}_{\fP}+1)) \left(72|A|^3\sqrt{2\ln (2d)} + 9\ln(\Abs{A}/\delta) + 18 - \lambda \right) \\
            &~~~+ \frac{1}{\sqrt{n}}\left(72|A|^3\sqrt{2\ln (2d)}+9\ln(\Abs{A}/\delta) + 9 \ln{8}\right)\\
            \leq & \frac{1}{\sqrt{n}}\left(72|A|^3\sqrt{2\ln (2d)}+9\ln(\Abs{A}/\delta) + 9 \ln{8}\right).
        \end{aligned}
    \end{equation}
    where the second inequality hold due to $\ln(4(\norm{\hat{\Theta}}_{\fP}+1)) > 1$ and $\norm{\hat{\Theta}}_{\fP}^2 \ln(4(\norm{\hat{\Theta}}_{\fP}+1)) + \ln 8 \geq \ln(4(\norm{\hat{\Theta}}_{\fP}+1))$,  last inequality due to $\lambda \geq 72|A|^3\sqrt{2\ln (2d)} + 9\ln(\Abs{A}/\delta) + 18$.
    And by Theorem~\ref{thm::PosterGenBoun} again, we have with probability at least $1-\frac{\delta}{2}$ over the choice of uniform distributed data $\fS:=\{s_i,a_i, s_i',r\}_{i=1}^{n}$,
    \begin{equation*}
        \begin{aligned}
            \fJ_{S,\lambda}(\tilde{\Theta})-\fR_{\fD}(\tilde{\Theta})
            & = \fR_{\fS}(\tilde{\Theta}) - \fR_{\fD}(\tilde{\Theta}) +\frac{\lambda}{\sqrt{n}}\norm{\tilde{\Theta}}_{\fP}^2\ln (\pi(\norm{\tilde{\Theta}}_{\fP}+1)) \\
            &\le \frac{\norm{\tilde{\Theta}}_{\fP}^2+1}{\sqrt{n}}\left(72|A|^3 \sqrt{2\ln (2d)}+ 9\left(\ln (4(\norm{\tilde{\Theta}}_{\fP}+1)) + \ln(\Abs{A}/\delta) \right)\right)\\
            &~~~+\frac{\lambda}{\sqrt{n}}\norm{\tilde{\Theta}}_{\fP}^2\ln (4(\norm{\tilde{\Theta}}_{\fP}+1)),
        \end{aligned}
    \end{equation*}
    and by $\norm{\tilde{\Theta}}_{\fP} \le \frac{24}{1-\gamma}\Abs{A}^{7/2}\norm{r(\cdot,\cdot)}_{\fB}$ we have
    \begin{equation}\label{eq::EmpPena&Emp}
        \begin{aligned}
            \fJ_{S,\lambda}(\tilde{\Theta})-\fR_{\fD}(\tilde{\Theta})
            &\le \frac{576\norm{r(\cdot,\cdot)}^2_{\fB}+1}{(1-\gamma)^2\sqrt{n}} \Abs{A}^7 \Biggl(72|A|^3\sqrt{2\ln (2d)} \\
            &~+9\left(\ln \left(\frac{96}{1-\gamma}\Abs{A}^{7/2}\norm{r(\cdot,\cdot)}_{\fB}+4\right) \ln(\Abs{A}/\delta) \right) \Biggl)\\
            &~+ \frac{576\lambda}{(1-\gamma)^2\sqrt{n}}\norm{r(\cdot,\cdot)}_{\fB}^2\Abs{A}^7\ln \left(\frac{96}{1-\gamma}\Abs{A}^{7/2}\norm{r(\cdot,\cdot)}_{\fB}+4\right).
        \end{aligned}
    \end{equation}
    Combining equations~\eqref{eq::GenErr}~\eqref{eq::Risk&EmpPena} and~\eqref{eq::EmpPena&Emp} together, we have the final result.
    % \begin{equation*}
    %     \begin{aligned}
    %         \fR_{\fD}(\hat{\Theta}) 
    %         & \leq \frac{1}{(1-\gamma)^2} \left( 256 \Abs{A}^{5} \frac{\norm{r(\cdot,\cdot)}_{\fB}^2}{m}  +   576 \gamma^2 C_T^2 (\Delta t)^2  \Abs{A}^{7} \norm{r(\cdot,\cdot)}^2_{\fB} \right) \\
    %         &~~~+\frac{576\norm{r(\cdot,\cdot)}_{\fB}^2+2}{(1-\gamma)^2\sqrt{n}}\Abs{A}^7(\lambda+1)\Biggl(72 |A|^3 \sqrt{2\ln (2d)} \\
    %         &~~~+ 9 \ln\left(\frac{96}{1-\gamma}\Abs{A}^{7/2}\norm{r(\cdot,\cdot)}_{\fB}+4 \right) +  9\ln(\Abs{A}/\delta) \Biggr).
    %     \end{aligned}
    % \end{equation*}
    % This complete the proof.
\end{proof}

\section{Conclusion}

This work provides an \emph{a priori} generalization error for continuous-time control problems with residual networks. We incorporate the discretization process into our modeling process,which is considering a discrete transition function. We assume that the transition function satisfies semi-group and Lipschitz properties, making the model aligned with the laws of object motion in the real world. Based on the assumptions, we propose a method to directly estimate the \emph{a priori} generalization error of the Bellman optimal loss. The crux of this method lies in the two transformation of the loss function and the using the neural network structure to replace the maximum operation. Our result can help ones in selecting the time step of discretization. In particular, our method does not require a boundedness assumption, which is closer to the practical applications' setting. In this work, we focus on deterministic environments and construct the Bellman optimal loss on a uniform distribution. We envision two extensions. First, we aim to extend our analysis techniques to stochastic environments and address the issue of biased estimation. Second, we plan to incorporate a sampling mechanism and dealing with distribution shift.

\section*{Acknowledgments}
This work is sponsored by the STI 2030-Major Project (New Generation Artificial Intelligence) under Grant 2021ZD0114203, the National Key R\&D Program of China  Grant No. 2022YFA1008200 (T. L.), the National Natural Science Foundation of China Grant No. 12101401 (T. L.), Shanghai Municipal Science and Technology Key Project No. 22JC1401500 (T. L.), Shanghai Municipal of Science and Technology Major Project No. 2021SHZDZX0102, and the HPC of School of Mathematical Sciences and the Student Innovation Center, and the Siyuan-1 cluster supported by the Center for High Performance Computing at Shanghai Jiao Tong University. We have used ChatGPT to revise the writing style and grammar.

%-------------------- 参考文献和附录 ---------------------
% References
%\printbibliography
%Bibliography
\bibliographystyle{unsrt}  
\bibliography{references}  

% Appendices
\end{document}